\ifdefined\XeTeXrevision\else\pdfoutput=1\fi
\documentclass[twoside,11pt]{article}

\usepackage{nonatbib}
\usepackagewithoutnatbib[preprint]{jmlr2e}
\usepackage{fixthmtools}
\usepackage{lastpage}

\usepackage[authoryear,citations,centerfigures,commands,enumerate,environments,lmrscinnershape,theorems]{AVT}

\crefname{equation}{}{}
\bibliography{Adversarial-TS}
\usepackage{framed}
\usepackage{xcolor}
\newenvironment{customleftbar}{\MakeFramed{\advance\hsize-\width \FrameRestore}\setlength{\parindent}{0pt}\slshape}{\endMakeFramed}

\title{Bayesian Algorithms for Adversarial Online Learning:\\from Finite to Infinite Action Spaces\expandafter\gdef\csname @fnsymbol\endcsname##1{1}\thanks{This manuscript was formerly titled \emph{An Adversarial Analysis of Thompson Sampling for Full-information Online Learning: from Finite to Infinite Action Spaces}.}\\[0.375\baselineskip]\normalfont\normalsize\slshape Dedicated to the memory of David Draper}

\author{%
  Alexander Terenin \\ 
  \addr Cornell University
  \AND
  Jeffrey Negrea \\ 
  \addr University of Waterloo and the Vector Institute
}

\jmlrheading{}{}{1-\pageref{LastPage}}{}{}{}{Terenin and Negrea}
\ShortHeadings{Bayesian Algorithms for Adversarial Online Learning}{Terenin and Negrea}

\editor{}

\begin{document}

\maketitle

\begin{abstract}
We develop a form Thompson sampling for online learning under full feedback---also known as prediction with expert advice---where the learner's prior is defined over the space of an adversary's future actions, rather than the space of experts. We show regret decomposes into regret the learner expected a priori, plus a prior-robustness-type term we call excess regret. In the classical finite-expert setting, this recovers optimal rates. As an initial step towards practical online learning in settings with a potentially-uncountably-infinite number of experts, we show that Thompson sampling over the $d$-dimensional unit cube, using a certain Gaussian process prior widely-used in the Bayesian optimization literature, has a $\mathcal{O}\Big(\beta\sqrt{Td\log(1+\sqrt{d}\frac{\lambda}{\beta})}\Big)$ rate against a $\beta$-bounded $\lambda$-Lipschitz adversary.
\end{abstract}

\section{Introduction} 

Making predictions in the face of adversarial feedback is a central problem in learning theory, and is known by a number of names, including online learning, no-regret learning, and prediction with expert advice \cite{cesabianchi2006prediction}.
In this framework, at each time step $t$, the \emph{learner} must decide on which action $x_t \in X$ to take, while, simultaneously and with knowledge of the learner's strategy, an \emph{adversary} chooses a reward function $y_t : X \to \R$ from some space of reward functions $Y$.
The learner then receives a reward of $y_t(x_t)$, and the game is repeated.
The learner's goal is to minimize their expected regret with respect to the best single action in hindsight.

Online learning is of fundamental importance because many other learning problems can be solved by reduction to an online learning oracle.
For example, in algorithmic game theory, coarse correlated equilibria can be computed using \emph{no-regret dynamics} \cite{roughgarden16}, which essentially work by implementing a set of agents that each play online learning algorithms against each other.
This approach extends to stricter equilibrium concepts when modified appropriately.
As a second example, \textcite{foster21,foster22,foster23} have recently shown that various variants of \emph{decision-making with structured observations}, a generalization of multi-armed bandits where pulling an arm can additionally provide information about other arms, and which includes certain forms of reinforcement learning, can be efficiently solved by reduction to an online regression oracle---a special case of an online learning~algorithm.
Other examples include generalization bounds \cite{lugosi23} and confidence sets \cite{clerico25}.

The aforementioned reductions, in general, involve a specialized action space $X$.
Online learning, however, is best-understood when $X$ is a finite set, with exponential weights variants the most-common algorithmic approach.
Motivated by these reductions, we ask:

\begin{customleftbar}
Given a general, potentially-uncountable action space $X$, and a space $Y$ of reward functions with a given degree of regularity, what practically-implementable algorithm should the learner~use?
\end{customleftbar}
In this work, we develop a framework for answering this question.
Our starting point will be the work of \textcite{gravin2016towards}, who derived the Nash equilibrium of infinite-horizon discounted online learning with three experts, namely $X = \{1,2,3\}$ and $Y = \{y \in \ell^\infty(X) : \norm{y}_\infty \leq 1\}$---and discovered the learner's optimal strategy to be a particular form of Thompson sampling: the learner should choose arms with probability proportional to them being best-in-hindsight under a certain maximin-optimal Bayesian prior distribution over possible adversaries.
While, at first, the scope of this result may seem very narrow, we show its minimax analogues hold in substantially more general settings.
To do so, we:
\1 Develop a framework for deriving regret bounds for Thompson sampling algorithms akin to those of \textcite{gravin2016towards}, but which allow general action spaces for the learner, and general priors over adversaries, as opposed to the maximin-optimal prior. We do so by showing that regret decomposes into the \emph{prior regret} the learner expects to incur, and \emph{excess regret} arising from the true adversary not matching the prior.
\2 Introduce an ansatz for selecting strong priors based on regret lower bounds. Specifically, we use Gaussian priors based on the covariance structure of a strong adversary to address our key question of what algorithm the learner should use. This approach can be interpreted as a recipe for transforming regret lower bounds into upper bounds.
\3 Develop techniques for bounding prior and excess regret, using expected supremum bounds for the former, and by showing the latter are bounded by certain Bregman divergences. These divergences are dual to those appearing in the \emph{local norm} analysis of \textcite{orabona2019modern},~and are closely related to certain terms arising in the \emph{stochastic smoothing} formalism of \textcite{abernethy2016perturbation}.
\4 Generalize the key technical inequality used by \textcite{abernethy14} for bounding certain Hessians associated with Gaussian perturbations, which requires independence across experts, to allow for general covariances.
Our argument is more flexible, and is probabilistic, as opposed linear-algebraic, in nature.
To our knowledge, this has been a key barrier to extending convex-analytic FTPL analyses to general action spaces.
\5 Show, for Gaussian process priors, that the aforementioned recipe can work in non-discrete action spaces: in the setting $X = [0,1]^d$ with uncountably-many experts and $Y$ the set of $\beta$-bounded $\lambda$-Lipschitz functions, Thompson sampling achieves a~regret~of
\[
\c{O}\del{\beta \sqrt{Td\log\del{1+\sqrt{d}\frac{\lambda}{\beta}}}}
.
\] 
\0 

Thompson sampling, named for the seminal work of \textcite{thompson1933likelihood,thompson1935theory}, is a special case of the \emph{follow-the-perturbed-leader} framework \cite{kalai2005efficient}, with perturbations given by the posterior variability of the best action.
The analysis of Thompson sampling in the stochastic bandit setting---for instance, following \textcite{russo2016information}---leverages the ability of the algorithm to borrow information across time to ensure posterior concentration, but not across actions.
In the adversarial setting, our work shows this can be addressed by having the learner place their prior over the \emph{set of adversary's reward sequences}, rather than on actions.
Posterior sampling has also been explored in general-feedback settings, for instance by \textcite{xu2023bayesian}, whose approach gives a variant of classical exponential weight algorithms in the full feedback setting we~study.

The algorithms we consider are can be viewed as a form of random rollouts.
From this view, our results are closely-related to the \emph{relax-and-randomize} perspective of \textcite{rakhlin2012relax}, but do not rely on Rademacher-complexity-theoretic assumptions on the perturbations, and lead to an easier-to-check condition relating the adversary's function class and perturbation distribution.
Our results can also be seen as a constructive analog of certain results of \textcite{sridharan2010convex}, which studies more-general games where the learner's strategies lie in spaces generalizing the space of probability measures we consider.

In spite of its completely differing motivation, the terms in our framework are closely related to those arising in the \emph{stochastic smoothing} approach of \textcite{abernethy2016perturbation} for analyzing \emph{follow-the-perturbed leader} algorithms---and, are also dual, in the convex conjugate sense, to those appearing in the \emph{local norm} analysis of \textcite{orabona2019modern} for more-general \emph{follow-the-regularized-leader} algorithms.
Our analysis reveals that learning rate schedule restrictions, which are used in both approaches---and exclude Thompson-sampling-like algorithms which correspond to increasing sequences of learning rates---are not necessary.

The algorithms we obtain from this approach can be implemented numerically on infinite state spaces in a manner similar to standard implementations of Thompson sampling used in Bayesian optimization \cite{wilson20,wilson21}.
In addition to our motivations arising from computational game theory, we therefore also view our work as a first step towards integrating ideas from adversarial learning into Bayesian optimization.

\section{Adversarial Full-information Online Learning in General Action Spaces}
\label{sec:preliminaries}
We study a variant of the \emph{online learning game}, first considered by \textcite{vovk1990aggregating,littlestone1994weighted} following work of \textcite{cover1966behavior}, and also called \emph{prediction with expert advice}.
This refers to a sequential two-player zero-sum game between a \emph{learner} and \emph{adversary}, both of which are allowed to randomize. 
At each time $t\in[T] = \cbr{1,..,T}$, the game proceeds as follows: 
\1 The learner picks $x_t\in X$, where $X$ is called the \emph{action space} or \emph{space of experts}.
\2 The adversary picks $y_t : X \to \R$, assumed to lie in some space $Y$ of \emph{reward functions}.
\3 Both choices are revealed and learner receives a \emph{reward} of $y_t(x_t)$ from the adversary.
\0
At time $T$, the learner's total rewards are adjusted by the total reward at the best-in-hindsight point to determine the game's total value.
From the adversary's perspective, the game's expected value is therefore given by the \emph{expected regret}
\[
R(p,q) &= \E_{\substack{x_t\~p_t\\y_t\~q_t}} \sup_{x\in X} \sum_{t=1}^T y_t(x) - \sum_{t=1}^T y_t(x_t)
\]
where $p = (p_1,..,p_t)$ is the learner's strategy, and $q = (q_1,..,q_t)$ is the adversary's strategy, both of which may be chosen adaptively.
The learner seeks strategies that minimize the expected regret, while the adversary aims to maximize it.
For a deterministic adversary, we write $R(p,y)$.
Using this formalism, the classical setting corresponds to taking $X = [N]$ to be a finite set, and $Y = \{y\in\ell^\infty(X;\R) : \norm{y}_\infty \leq 1\}$ to be an $\ell^\infty$ unit ball. 
In more general settings, to ensure the game's values are well-defined, we assume $X$ is second-countable compact Hausdorff, and $Y\subseteq C(X;\R)$ is a convex subset of continuous real-valued functions on $X$. 
For further details, see~\Cref{apdx:technical}.

In this work, we seek algorithms whose rate of regret, as a function of the time horizon $T$ and the complexity of the adversary's decision space $Y$, matches the \emph{minimax regret} 
\[
\min_p \max_q R(p,q)
.
\]
For finite expert classes, the minimax expected regret is $\Theta(\sqrt{T\log N})$. 
In this setting, the most well-known strategy achieving the minimax rate is the \emph{exponential weights} algorithm, also known as \emph{hedge} \cite{littlestone1994weighted,vovk1990aggregating}, which processes the historical rewards in order to explicitly compute a suitable probability measure over experts.

To generalize this to infinite expert sets, one approach is to represent the learner's actions by a density with respect to some prior \cite{maillard2010online,alquier21a,negrea2021minimax}: this typically yields policies for which sampling is intractable.
Another approach---which is closely related to, but does not directly cover, the infinite-expert setting---is to allow the learner and adversary to both act on the same Hilbert space, where much of the classical finite-dimensional online mirror descent results remain valid \cite{joulani17modular}. 
Part of our aim will be to develop appropriate convex-analytic foundations which serve the same purpose, but remain valid in genuinely non-Hilbert settings, such as the space of probability measures the learner's strategies are defined in.

A third approach focuses on proving minimax rates by reduction to the finite-expert setting via, for instance, sequential covers \cite{rakhlin2015online}. 
This establishes information-theoretic limits of sequential prediction using metric entropy variants.
From a computational viewpoint, such approaches ultimately require one to sample from complicated probability distributions on either infinite or large finite sets, which is only feasible in low-dimensional settings.
In certain cases, efficient reductions can be achieved with zooming-based algorithms \cite{kleinberg2019bandits}, which exploit low dimensionality to construct discrete structures within the continuous space.

Another class of algorithms known to achieve the minimax rate of expected regret in the finite-expert setting are \emph{follow-the-perturbed-leader (FTPL)} strategies, such as \textcite{kalai2005efficient,devroye2013prediction,vanerven14,abernethy2016perturbation}, and the relax-and-randomize framework of \textcite{rakhlin2012relax}.
These operate implicitly: rather than an explicit formula to compute the probabilities of a mixed strategy, a sampling procedure is defined directly, by adding a carefully-constructed random perturbation to the observed rewards.
In general action spaces, the key challenge is to understand how to translate smoothness assumptions about the adversary's reward function class into explicit choices of the perturbation distribution.

\section{Thompson Sampling for Adversarial Full-information Online Learning}
\label{sec:general_results}

In this work, we develop algorithms which are a form of \emph{Thompson sampling}---a Bayesian approach to decision-making under uncertainty widely-used in the stochastic bandit literature, which works by sampling each action according to the posterior probability it is optimal.
At first, one might not expect a Bayesian approach to make any sense for online learning: unlike in statistical learning, the reward functions $y_t$ in online learning vary arbitrarily at each round, so there is no well-specified likelihood one can use to build a Bayesian model.
Moreover, the most-common online learning algorithm---namely, exponential weights---uses an update rule which resembles, but does not exactly correspond, to classical Bayesian learning: it includes a learning rate term, which makes it trust the data~less.

In light of this, somewhat surprisingly, \textcite{gravin2016towards} show that the exact Nash equilibrium of discounted online learning with three experts admits a Bayesian interpretation: the learner's algorithm is a form of Thompson sampling.
At each round, the learner plays each expert according to its conditional probability of being the best eventual expert, with the distribution over future experts given by the optimal adversary of the game's maximin dual.
The proof proceeds by calculating the maximin-optimal adversary explicitly, then using its form to derive the learner's strategy.

Since the maximin optimal adversary's actions over different rounds are not independent, such adversaries are intractable in more general settings.
However, the functional form of the learner's strategy motivates the following question: \emph{does Thompson sampling, with respect to a prior given by a strong adversary, lead to a minimax rate regret bound?}
To study this, we begin by defining the Thompson sampling strategy for prediction with expert advice.

\begin{restatable}[Thompson Sampling]{definition}{DefTS}
\label{def:ts}
Let $q^{(\gamma)}$ be a prior over the adversary's strategy.
Define the \emph{Thompson sampling} strategy $p_t$ for the learner to be the strategy that (i) draws one sample from the conditional reward sum $\sum_{\tau=1}^T \gamma_\tau \given \gamma_1 = y_1, .., \gamma_{t-1} = y_{t-1}$, and (ii) plays the best-in-hindsight point with respect to this sample, with ties broken according to a given tie-breaking rule.
\end{restatable}

To further understand \Cref{def:ts}, note that this strategy can equivalently be written
\[
\label{eqn:ts}
x_t &= \argmax_{x\in X} \sum_{\tau=1}^{t-1} y_\tau(x) + \sum_{\tau=t}^T \gamma_\tau(x)
&
(\gamma_t,..,\gamma_T) &\~ q^{(\gamma)} \given \gamma_1 = y_1, .., \gamma_{t-1} = y_{t-1}
\]
by plugging in the values observed thus far into the reward sum.
It follows that this specific form of Thompson sampling---defined over the adversary's future reward functions---is a follow-the-perturbed-leader strategy, where the perturbation distribution and learning rate are set according to the posterior over the adversary's future actions.
\Cref{def:ts} can therefore be implemented, even for non-discrete spaces, given access to suitable Gaussian process sampling and optimization oracles.
To distinguish $q^{(\gamma)}$ from the true adversary, throughout this work we refer to it as the \emph{virtual adversary} under the prior.

Just as the key question for FTPL involved selecting the perturbation distribution, the key question for Thompson sampling is to understand which priors lead to regret guarantees achieving the minimax rate.
To develop this understanding, our angle of attack will be to extend the notions used by \textcite{gravin2016towards} to our more general setting.

\subsection{A Bayesian Decomposition of Minimax Regret: from lower to upper bounds}
\label{sec:excess-regret}

We now study minimax regret of Thompson sampling, as in \Cref{def:ts}.
This algorithm admits a natural Bayesian interpretation, so we will seek to analyze it in a Bayesian way, with the aim of extending the theoretical picture painted by \textcite{gravin2016towards} to the much broader minimax regret setting.
Let $\Gamma_t$ be the FTRL regularizer of Thompson sampling: defined, following \textcite[Chapter 30.5]{lattimore20}, in terms of its convex conjugate $\Gamma^*_t(f) = \E \sup_{x\in X}\del{ f(x) + \gamma_{t:T}(x)}$, where we denote sums by $\gamma_{1:T} = \sum_{t=1}^T \gamma_t$ with the convention $\gamma_{T+1:T} = 0$. 
We also write $\pair{y_t}{p_t} = \E_{x\~p_t} y_t(x)$.
We now state~our~first~result.
\begin{restatable}{proposition}{PropRegretDecomposition}
\label{prop:excess_regret}
We have
\[
R(p,y) = R(p,q^{(\gamma)}) + \ubr{\sum_{t=1}^T \Gamma^*_{t+1}(y_{1:t}) - \Gamma^*_t(y_{1:t-1}) - \pair{y_t}{p_t} + \E\pair[0]{\gamma_t}{p^{(\gamma)}_t}}{E_{q^{(\gamma)}}(p,y)}
\]
where $p^{(\gamma)}_t$ is a copy of the learner's strategy applied to simulated data $\gamma_{1:t-1}$ drawn from the prior $q^{(\gamma)}$ instead of the $y_{1:t-1}$. 
We call $ R(p,q^{(\gamma)})$ the \emph{prior regret} and $E_{q^{(\gamma)}}$ the \emph{excess regret}.
\end{restatable}

All proofs are in \Cref{apdx:technical}. 
A sketch is as follows: the idea is to apply a standard telescoping argument, similar to the more general follow-the-regularized-leader (FTRL) regret decomposition given by \textcite[Lemma 7.1]{orabona2019modern}, with two key modifications: we replace the FTRL regularizer with its convex conjugate, and additionally include the $\gamma_t$-terms given by the virtual adversary induced by the prior.
The claim follows by noting that the telescoping-sum boundary term, $\Gamma^*_t(0)$, is by definition the expected best-in-hindsight value under the virtual adversary, and~therefore its difference with $\pair[0]{\gamma_t}{p^{(\gamma)}_t}$ is equal to the expected regret incurred against the strategy $q^{(\gamma)}$.

The advantage of this decomposition is that (a) it is completely general, requiring no assumptions about the setting, virtual adversary distribution, or learning rates, and (b) even in non-discrete settings, it is not difficult to choose $q^{(\gamma)}$ in a way that makes $R(p,q^{(\gamma)})$ straightforward to upper-bound.
To see this, we now explore consequences of this decomposition: to do so, we introduce a criterion central to the work of \textcite{gravin2016towards} that a strong virtual adversary should satisfy.

\begin{restatable}{definition}{DefEqualizing}
We say $q^{(\gamma)}$ is \emph{equalizing} if, for every $t$, the conditional expectation of $\gamma_t(x)$ given the history up to time $t-1$ is almost surely constant in $x$, and call it \emph{centered} if it is equal to~zero.
\end{restatable}

\textcite{gravin2016towards} refer to equalizing adversaries as \emph{balanced}: we instead use the more-standard term from game theory---see \textcite[Chapter 11]{straffin2010game}.
To ease notation in what follows, we prove that if $q^{(\gamma)}$ is independent across time, we can take $q^{(\gamma)}$ to be centered without loss~of~generality.

\begin{restatable}{lemma}{LemCentering}
Given an online learning game and an equalizing adversary $q^{(\gamma)}$ supported on $Y$ which is independent across time, there exists an equivalent online learning game with centered equalizing adversary $q^{(\gamma')}$ supported on $Y'$.
\end{restatable}

If the adversary plays an equalizing strategy, the learner's expected regret does not depend on their choice of algorithm $p_t$.
Thus, the existence of an equalizing adversary $q^{(\gamma)}$ implies a regret lower bound---a view we now formalize.

\begin{restatable}{proposition}{PropLowerBound}
\label{prop:lower_bound}
Let $q^{(\gamma)}$ be equalizing, centered, and supported on at most $Y$. Then 
\[
\label{eqn:lower_bound}
\E \sup_{x\in X}\sum_{t=1}^T \gamma_t(x) = R(\.,q^{(\gamma)}) \leq \min_p \max_q R(p,q)
.
\]
\end{restatable}

This statement merits two observations:
\1 In choosing a prior for Thompson sampling, we should define a \emph{strong} adversary to be one which is equalizing, and for which the inequality in \Cref{eqn:lower_bound} is tight up to a constant.
\2 For a strong virtual adversary in this sense, if $E_{q^{(\gamma)}}(p,y)$ is of the same order as $R(p,q^{(\gamma)})$ for all $y$, then Thompson sampling achieves the minimax rate.
\0 
\Cref{prop:excess_regret,prop:lower_bound} therefore provide an ansatz for defining good Thompson sampling priors, and a blueprint for transforming regret lower bounds into upper bounds.
To verify that such a blueprint can work in principle, consider the discrete setting $X = [N]$ and $Y = \{y\in\ell^\infty(X) : \norm{y}_\infty \leq 1\}$.
Taking $q^{(\gamma)}$ to consist of IID Rademacher random variables for each expert leads to $R(\.,q^{(\gamma)}) = \Theta(\sqrt{T\log N})$ by standard arguments---indeed, tight upper bounds on $R(\.,q^{(\gamma)})$ can be proven even in infinite-expert settings by chaining techniques \cite{talagrand2005generic}.
Thus, if one can prove that $E_{q^{(\gamma)}}(p,y) \leq \c{O}(\sqrt{T\log N})$---or, in more general settings, whatever the corresponding rate is---the aforementioned lower bound transforms into a minimax-rate upper bound.

\subsection{A Bregman Divergence Bound on Excess Regret}

In the preceding section, we developed a general decomposition of Thompson sampling's minimax regret.
The terms in this decomposition can be interpreted as (i) the regret the learner expects to incur under the prior, and (ii) an excess regret term, which quantifies how much worse---or better---the learner will do when faced against the true adversary.
Bounding term (i) is straightforward, so we now focus on term (ii).
Taking inspiration from the finite-dimensional setting where FTPL---and hence Thompson sampling---can be understood as FTRL with regularizer sequence $\Gamma_t$, for which the respective Fenchel conjugates are the functions $\Gamma^*_t$ defined previously, we now prove that the excess regret can be bounded term-by-term by Bregman divergences $D_{\Gamma^*_t}$ defined with respect to $\Gamma^*_t$.

\begin{restatable}{proposition}{PropBregmanBound}
Assume $q^{(\gamma)}$ is equalizing, independent across time, and ties almost never occur.~Then
\[
E_{q^{(\gamma)}}(p,y) \leq \sum_{t=1}^T D_{\Gamma^*_t}(y_{1:t} \from y_{1:t-1})
.
\]
\end{restatable}

Here, \emph{ties} refer to situations where the optimization problem defining Thompson sampling admits more than one optimum---a probability-zero event in the examples we consider.
To show this result, the idea is to (a) add-subtract $\pair{y_t - \gamma_t}{p_{t+1}}$, (b) use the tower rule to write $\Gamma^*_t(\.) = \E_{\gamma_t\~q^{(\gamma)}_t} \Gamma^*_{t+1}(\. + \gamma_t)$ to ensure $\Gamma^*$-functions for the same summand have matching time indices, (c) use Bregman duality to rewrite the resulting terms as a Bregman divergence in primal form (up to technicalities), then (d) apply Young's inequality for the convex function $p \mapsto D_{\Gamma_{t+1}}(p \from p_{y_{1:t-1} + \gamma_t})$, thereby canceling out the primal Bregman divergence, which is otherwise difficult to understand because $\Gamma_t$ is only defined implicitly.

The advantage of this bound is that it is completely general, requiring minimal assumptions on the setting and virtual adversary.
The disadvantage is that it is non-negative, whereas the excess regret can be negative for a sequence of reward functions $y_t$ which is not worst-case.
We believe it might be possible to relax the assumptions further, at the cost of a more-involved analysis: we omit this, given our aim of understanding non-discrete settings, and because the virtual adversaries we consider satisfy these assumptions.

One can interpret the resulting Bregman divergences probabilistically by explicitly writing out the $\Gamma^*_t$-terms.
Introducing the notation $x^*_f = \argmax_{x\in X} f(x)$, with ties broken according to the tie-breaking rule chosen as part of Thompson sampling's definition, this~gives
\[
\sum_{t=1}^T D_{\Gamma^*_t}(y_{1:t} \from y_{1:t-1}) = \E \sum_{t=1}^T \del{(y_{1:t} + \gamma_{t:T})(x^*_{y_{1:t} + \gamma_{t:T}}) - (y_{1:t} + \gamma_{t:T})(x^*_{y_{1:t-1} + \gamma_{t:T}})}
\]
where each summand is the difference between the posterior mean total reward of the random best expert, and posterior mean total reward of the best expert if $y_t$ were replaced with $0$---this is analogous to the comparison between the be-the-regularized-leader and follow-the-regularized leader interpretations of FTRL's regret bounds.
This expression should be contrasted with the excess regret itself, which can analogously be written
\[
E_{q^{(\gamma)}}(p,y) = \E \sum_{t=1}^T \del{(y_{1:t} + \gamma_{t+1:T})(x^*_{y_{1:t} + \gamma_{t+1:T}}) - (y_{1:t} + \gamma_{t:T})(x^*_{y_{1:t-1} + \gamma_{t:T}})}
\]
and whose interpretation is similar, but where the function being evaluated in the first term no longer matches that of the second term.
This seemingly-subtle difference, from a probabilistic view, has significant consequences from a convex-analytic one: using the integral form of Taylor's Theorem for convex functions, one can write the resulting Bregman divergences as
\[\label{eq:bregman-hessian-bd}
D_{\Gamma^*_t}(y_{1:t} \from y_{1:t-1}) = \frac{1}{2} \int_0^1 \p^2_{y_t,y_t} \Gamma^*_t(y_{1:t-1} + \alpha y_t) \d\alpha
\]
where $\p^2_{(\.,\.)} \Gamma^*_t$ is the (distributional) second Gâteaux derivative.
This expression is an integral analog of certain terms arising in the local norm analysis of \textcite{orabona2019modern}, with one key difference: our Hessian terms are defined with respect to the convex conjugate $\Gamma^*_t$ instead of the regularizer $\Gamma_t$.
To get an intuitive sense for this difference, consider what would occur for $\Gamma(p) = D_{\f{KL}}(p \from p_0)$.
Then, for most signed measures, the first Gâteaux derivative $\p_u \Gamma(p)$ would not always be finite-valued, and one would need to restrict to signed measures with explicit density bounds.
In contrast, $\Gamma^*(f) = \log \E_{x\~p_0} \exp f(x)$, for which $\p_v \Gamma^*(f)$ is finite-valued for all bounded measurable $f,v$.
Given this behavior for first derivatives, one anticipates that Hessians of the convex conjugate $\Gamma^*_t$ will be easier to work~with in this case.

\subsection{Adversarial Guarantees via Gaussian Process Priors}
\label{sec:gaussian_adversary}

The preceding sections provide a Bayesian way of thinking about the online learning game, which yields a recipe for constructing concrete online learning algorithms: choose a strong equalizing adversary as the prior.
Such a virtual adversary should reflect the function class $Y$, and, at the same time, be amenable to analysis.

A natural starting point for constructing a strong adversary is, at every time point, to randomly sample from a set of worst-case reward functions supported at the extreme points of $Y$.
Developing this approach naturally leads to Rademacher-complexity-theoretic lower bounds via sequential covers \cite{sridharan2010convex,rakhlin2012relax}, but not directly to an analytically-tractable Hessian for \Cref{eq:bregman-hessian-bd}.

Assuming $Y$ consists of bounded functions, observe that under mild regularity conditions, the random sum obtained in this manner converges to a Gaussian process, which we show in the sequel does have an analytically-tractable Hessian.
The expected suprema of sums of independent Gaussians, in turn, are known to be of the same order as that of independent Rademachers, suggesting that a Gaussian virtual adversary constructed as a limit of strong virtual adversaries should remain strong.
We therefore adopt Gaussian process virtual adversaries $\gamma_t\~[GP](0,k)$, where $k$ is the covariance kernel, as our ansatz: to analyze their performance, we begin with prior regret.

\begin{restatable}{lemma}{LemPrior}
Let $q^{(\gamma)}_t = \f{GP}(0,k)$, assumed IID over time.~Then
\[
R(\cdot, q^{(\gamma)}) = \E\sup_{x\in X}\sum_{t=1}^T \gamma_t(x) = \sqrt{T} \E\sup_{x\in X} \gamma_1(x)
.
\]
\end{restatable}

Note that, in infinite dimension, a Gaussian virtual adversary constructed as a limit of random sums supported on $Y$ in general \emph{will not} be supported on $Y$.
As a result, such an adversary does not immediately certify a regret lower bound---though in the vast majority of cases, we expect a lower bound of the same rate to follow.
We note that lower bounds on suprema of Gaussian processes can be derived generically using Fernique--Sudakov minorization, while lower bounds of general random processes can require specialized techniques.

We now proceed to study excess regret, beginning with writing down the Gâteaux Hessian. 
By applying Gaussian integration by parts, which one can derive in the infinite-dimensional setting from the Cameron--Martin Theorem, one can show
\[
\p^2_{u,v}\Gamma^*_t(f) = \frac{1}{\sqrt{T-t+1}} \E u(x^*_{f+\gamma_{t:T}}) \pair{\gamma_{t:T}}{\c{K}^{-1} v}
\]
where $\c{K}^{-1}$ is the inverse (on appropriate subsets) of the \emph{covariance operator} $\c{K} : \c{M}_s(X) \to C(X;\R)$ defined by $\mu\to \int_X k(x,\.) \d\mu(x)$, and also known as the \emph{kernel mean embedding}.
Rather than making this expression precise in the literal sense, to ease notation and reduce the need for heavy analytic machinery, our argument will first pass to a finite cover $X_h \subseteq X$, where $h>0$ is the radius.
This enables us to work with finite-dimensional matrices: after bounding the resulting expressions, we then take a limit $h\to0$, obtaining the same rates as if we had worked directly in infinite dimension---but without the need to prove invertability of $\c{K}^{-1}$ or deal with other measure-theoretic subtleties.
We emphasize this discretization is merely an analytic tool for studying the continuous algorithm, in contrast to other methods that operate by applying discrete algorithms on a discretization of a continuous space. 
Using \Cref{lem:cover}, the Hessian can be approximated by its analog on a cover, namely
\[
\grad^2 \Gamma^*_{h,t}(f) = \frac{1}{\sqrt{T-t+1}}\E \1_{i^*_{f_h + \gamma_{h,t:T}}} \gamma_{h,t:T}^T K^{-1}_h
\]
where $\Gamma^*_{h,t}$ is a finite-dimensional analog of $\Gamma^*_t$, the matrix $K_h$ is defined by $K_{ij} = k(x_i, x_j)$ for elements $x_i$ and $x_h$ of the cover, and $\1_j$ is a vector which is $1$ at entry $j$ and zero everywhere else, and $i^*_{f_h + \gamma_{h,t:T}}$ is the maximizer of the respective sum on the cover.
For the case $K=I$, \textcite{abernethy14} bounds $y^T_h \grad^2 \Gamma^*_{h,t}(f) y_h$ through a specialized inequality that relates the $L(\infty,1)$-norm of the Hessian to its trace.
To our knowledge, no suitable extension of this inequality to the case where $K$ is not a multiple of the identity is known, and almost every other approach we tried---without an appropriate analogue---leads to bounds which are typically loose by a factor of $|X_h|$, and therefore achieve vastly suboptimal rates.

To overcome this, the idea will be to seek a probabilistic proof of a comparable inequality, rather than a linear-algebraic one.
At a high level, the idea is to condition on the maximizer index $i^*_{f_h+\gamma_{h,t:T}}$, and apply the Tower Rule: one can show conditioning both on the argmax and max results in an inner conditional expectation representing a truncated multivariate normal distribution, which takes a certain explicit form.
By manipulating these conditional expectations in the right manner, we arrive at a simple condition which relates the kernel $k$ to the adversary's function class $Y$, and leads to a bound on $y^T_h \grad^2 \Gamma^*_{h,t}(f) y_h$.
Taking a limit as $h\to0$ to pass back to the infinite-dimensional world, we obtain the following general bound on the Bregman divergence.

\begin{restatable}{theorem}{ThmHessianBound}
\label{thm:bregman_hessian_bound}
Suppose that $q^{(\gamma)} = \f{GP}(0,k)$ is IID over time, where $k$ is continuous, strictly positive definite, and gives rise to a global expected modulus of continuity.
Suppose there is a constant $C_{Y,k}>0$ such that for any $x,x'\in X$ we have
\[
\sup_{y\in Y} y(x) - y(x')\frac{k(x,x')}{k(x',x')} \leq C_{Y,k} \del{1 - \frac{k(x,x')}{k(x',x')}}
.
\]
Then, letting $\xi_{t,\theta} = x^*_{y_{1:t-1} + \theta y_t + \gamma_{t:T}}$, we have
\[
D_{\Gamma^*_t}(y_{1:t} \from y_{1:t-1})\leq \frac{1}{2\sqrt{T-t+1}} \int_0^1 \E\del{ \del{y(\xi_{t,\theta})^2 + C_{Y,k} |y(\xi_{t,\theta})|}\frac{\gamma_{t:T}(\xi_{t,\theta})}{k(\xi_{t,\theta},\xi_{t,\theta})} }\d\theta
.
\]
\end{restatable}

This bound simplifies considerably in the case that $Y$ contains at most uniformly $\beta$-bounded functions.
In this situation, one can take $k$ to be of equal variance, and immediately obtain the following general bound on the regret of Thompson sampling under a Gaussian process prior, in the sense of \Cref{def:ts}, as a strategy for the online learning game.

\begin{restatable}{theorem}{ThmTS}
\label{thm:ts}
Let $X \subseteq \R^d$ be compact, and let $Y$ consists of bounded functions with $\norm{y}_\infty \leq \beta$.
Let $q^{(\gamma)} \~[GP](0,k)$ be IID over time, where $k$ satisfies the assumptions of \Cref{thm:bregman_hessian_bound}, and $k(x,x) = \sigma^2$ for all $x$.
Then Thompson sampling achieves a regret of
\[
R(p,q) \leq \sqrt{T} \del{1 + \frac{\beta(\beta + C_{Y,k})}{\sigma^2}} \E \sup_{x\in X} \gamma_1(x)
.
\]
\end{restatable}

To conclude, we note that the same argument---in spite of its Bayesian motivation---also applies to follow-the-perturbed leader with a constant learning rate, with an improved constant by a factor of $\sqrt{2}$.
The Bayesian viewpoint developed here, by leading to an ansatz for selecting good perturbation distributions, therefore also allows one to obtain bounds for more-classical convex-analytic algorithms akin to those of \textcite{kalai2005efficient,abernethy14,abernethy2016perturbation}.
We proceed to apply this bound to several spaces of interest.

\section{Examples}

In the preceding section, we established a general framework for reasoning about the minimax regret of Thompson sampling, with a prior given by a general virtual adversary $q^{(\gamma)}$.
We now use this framework to analyze online learning in two settings: (i) the classical finite-expert setting, in order to verify that our analysis of Thompson sampling obtains the correct rates, and (ii) a setting with uncountably many experts defined on a given interval, with an adversary restricted to bounded Lipschitz functions.
For the latter setting, to our knowledge, all known FTPL algorithms use covers as part of their computation, leading to exponential-in-dimension computational costs and extra $\sqrt{\log T}$ factors in the regret.

\subsection{Finitely Many Experts}
\label{sec:finite-dimensions}

We now apply our theory to the classical finite-expert setting, in order to verify that it produces a result which at least matches established theory.
Let $X = [N]$ and $Y = [-1,1]^N$.
By standard lower-bound theory \cite{cesabianchi2006prediction}, the equalizing adversary which at each time $t$ plays independent Rademacher random variables in each coordinate leads to a rate of $\Omega(\sqrt{T\log N})$. Since known algorithms---such as exponential weights and various FTPL variants---achieve this rate, it is the minimax rate, and therefore fits our criterion for a strong adversary.

By approximating the IID Rademacher adversary by its asymptotic Gaussian limit, we obtain a strong virtual adversary for which the bounds of \Cref{sec:gaussian_adversary} apply.
To state these, we first describe the prior regret term.
Since $\frac{1}{T} \sum_{t=1}^T \gamma_t \~[N](0, \sigma^2 I_N)$, using the standard maximal inequality for equal-variance Gaussians, we obtain
\[
R(\., q^{(\gamma)}) = \E \max_{x\in X}\sum_{t=1}^T \gamma_t(x) \leq \sigma\sqrt{2 T\log N}.
\]
which is sharp up to small constants.
Next, we focus on the excess regret. 
In this setting, we have $k(x,x') = 0$ for $x \neq x'$, and therefore the required inequality reduces to 
\[
\sup_{\norm{y}_\infty\leq 1} y(x) \leq 1 = C_{Y,k}
\]
which holds trivially.
Our theory therefore yields the following bound.

\begin{restatable}{corollary}{CorDiscreteRegret}
Suppose $X = [N]$, $Y = \{y\in \ell^\infty : \norm{y}_\infty \leq 1\}$, and $k$ is represented by the diagonal matrix $\sigma^2 I$ with $\sigma = \sqrt{2}$.
Then Thompson sampling achieves a minimax regret of
\[
R(p,q) \leq 4\sqrt{T\log N}
.
\]
\end{restatable}

Thompson sampling therefore achieves the desired minimax rate, even if there are more than three experts and we are working in the finite-horizon setting.
This shows that, even though the game-theoretic analysis of \textcite{gravin2016towards} may at first seem very limited---applying only to $X = \{1,2,3\}$, to the infinite-horizon discounted setting, and to exact Nash equilibria---the key idea of using a perturbation sequence derived from an application of Bayes' Rule continues to work with more experts, up to approximately-optimal priors and approximate equilibria arising from a weaker solution concept given by minimax rates.

From a technical standpoint, a key difference between our framework and classical results of \textcite{abernethy2016perturbation} is that we apply Young's inequality to obtain the Bregman divergence, whereas they obtain the same term by other means.
This seemingly-minor difference leads \textcite{abernethy2016perturbation} to impose non-increasing learning rate restrictions, which preclude our Thompson-sampling-based approach from being covered by their framework.

If we abandon the Bayesian interpretation, switch to working with follow-the-perturbed leader with a constant learning rate, and modify our framework appropriately, applying our results instead yields an inequality of $R(p,q) \leq 2\sqrt{T\log N}$ which matches the bounds of \textcite{abernethy2016perturbation}, including constants.
This universal constant gap vanishes if one instead considers the infinite-horizon discounted setting, where Thompson sampling and FTPL with optimal tuning coincide.

\subsection{Experts in \texorpdfstring{$[0,1]^d$}{[0,1]d} with a Lipschitz Adversary}
\label{sec:bounded-Lipschitz}

We now turn our attention to a prototype setting with uncountably many experts.
Our aim is to develop a proof-of-concept analysis which uses the Bayesian approach to analyzing minimax regret developed in \Cref{sec:excess-regret} to obtain a non-vacuous guarantee, as a first step towards a comprehensive understanding of online learning in general action spaces.
For this, we consider the space $X = [0,1]^d$ with an adversary playing reward functions that lie in the subset $Y = \{y\in\f{BL}(X;\R) : \norm{y}_\infty \leq \beta, \abs{y}_{\f{Lip}} \leq \lambda\}$ of the Banach space of bounded Lipschitz functions, where the adversary's allowed functions are $\beta$-bounded and $\lambda$-Lipschitz.

We work with a virtual adversary given by a Gaussian process $q^{(\gamma)}_t = \f{GP}(0,k)$  with kernel $k$.
With access to an optimization oracle, one could implement this form of Thompson sampling in practice---indeed, the resulting algorithm would be similar to well-developed variants of Bayesian optimization \cite{frazier18,garnett23}, where how to perform Thompson sampling numerically is well-understood \cite{wilson20,wilson21}.

What regret guarantees might Thompson sampling have in this setting? 
The Bayesian viewpoint suggests we should choose our prior to be a Gaussian process whose kernel is obtained by asymptotically approximating a strong adversary with a Gaussian process.
If the adversary plays bounded Lipschitz functions, a strong virtual adversary can be obtained by considering zig-zag-shaped functions where the height of each individual spike is Rademacher-distributed.
In the same way a Brownian motion arises as a scaling limit of random walks with discrete steps of $\pm 1$, which correspond to our Lipschitz constraint up to scaling---and due to the connection between Brownian motion and the Ornstein--Uhlenbeck process, which includes a suitable drift term to ensure it is shift-invariant---one can expect the covariance arising from a scaling limit of the aforementioned virtual adversary to be Matérn-1/2.
For this choice, with variance $\sigma^2$ and length scale $\kappa$, the covariance kernel is $k(x_1,x_2) = \sigma^2 \exp\del{-\frac{\abs{x_1-x_2}}{\kappa}}$. 
By \Cref{lem:matern_sup} of \Cref{apdx:technical}, for this family of kernels on $X=[0,1]^d$, our prior regret is
\[
R(p,q) \leq 16 \sigma \sqrt{Td\log\del{1+\frac{\sqrt{d}}{\kappa}}}
.
\]
In the aforementioned random-walk scaling limit, the asymptotic variance is $\Theta(\beta^2)$ and the asymptotic length scale is $\Theta(\lambda^{-1})$, suggesting a lower bound on the minimax regret of $\Omega\del{\beta \sqrt{Td\log(1+\sqrt{d}\frac{\lambda}{\beta})}}$.
In the interest of brevity, we omit a rigorous derivation of this lower bound, and instead view it a reasonable motivation for our choice of prior.

We now apply the theory of \Cref{sec:gaussian_adversary}.
The key inequality for \Cref{thm:bregman_hessian_bound} is 
\[
\sup_{y\in Y} y(x) - y(x') \frac{k(x,x')}{k(x',x')} \leq \ubr{\frac{\beta\del{\lambda + \frac{1}{\kappa}}}{\frac{1}{\kappa} \del{1 - e^{-\frac{2}{\lambda\kappa + 1}}}}}{C_{Y,k}} \del{1 - \frac{k(x,x')}{k(x',x')}}
\]
which follows by bounding the left-hand-side with a worst-case Lipschitz estimate, and then bounding the resulting estimate by the kernel, where the constant arises because the function class depends on $\beta$ and $\lambda$, while the kernel depends on $\sigma$ and $\kappa$.
We obtain~the~following.

\begin{restatable}{corollary}{CorBLRegret}
Suppose that $X = [0,1]^d$, $Y = \{y\in\f{BL}(X;\R) : \norm{y}_\infty \leq \beta, \abs{y}_{\f{Lip}} \leq \lambda\}$, and $k$ is Matérn-1/2 with variance $\sigma^2$ and length scale $\kappa$.
Then Thompson sampling with parameters $\sigma^2 = \beta$ and $\kappa = \frac{1}{\lambda}$ achieves a regret of 
\[
R(p,q) \leq \beta \del{32 + \frac{32}{1 - \frac{1}{e}}}\sqrt{T d\log\del{1 + \sqrt{d}\frac{\lambda}{\beta}}}
.
\]
\end{restatable}

Thompson sampling therefore achieves what we expect to be the minimax rate, even in this setting, where there are infinitely-many experts.
This shows that the key algorithmic approach motivated by the three-expert Nash equilibrium derived by \textcite{gravin2016towards}---namely, that the learner should play an FTPL strategy, where the perturbation sequence is derived by applying Bayes' Rule over the space of adversary's actions---continues to work, even in the presence of smoothness.
We can intuitively sanity-check our obtained rate by embedding the finite expert setting with $N$ experts into $(\beta,\lambda)=(1,N)$: the bound would then achieve the minimax rate for the embedded problem.

\section{Discussion} 

In this paper, we developed a framework for deriving minimax regret bounds for Thompson sampling algorithms akin to the Nash equilibrium strategy of \textcite{gravin2016towards} for the three-expert setting, but which allow a general, potentially-uncountable number of experts.
The framework decomposes regret into the \emph{prior regret} the learner expects to incur, and an \emph{excess regret} term describing how much additional regret can be incurred due to the true adversary not matching the~prior.

Using this decomposition, we introduced an ansatz for constructing priors that one can expect lead to good minimax performance, by demonstrating that equalizing adversaries which lead to meaningful regret lower bounds are good candidates for Thompson sampling priors---thereby, deriving a general proof strategy for transforming regret lower bounds into regret upper bounds.

We carried out this this routine using an asymptotic approximation of the virtual adversary, where the prior is chosen to be Gaussian, thereby (i) obtaining minimax rates in the classical discrete setting, and (ii) obtaining a regret bound one can reasonably expect to be minimax in the setting where the adversary is constrained to bounded Lipschitz functions on the uncountable expert set~$X = [0,1]^d$.

Our work connects the game-theoretic perspective of \textcite{gravin2016towards} with the analysis techniques of \textcite{abernethy14,abernethy2016perturbation,orabona2019modern}---and, in the process, removes learning rate schedule restrictions from these approaches, which would prohibit the analysis of Thompson sampling and related algorithms which correspond to FTPL with increasing learning~rates. We also extend the techniques for bounding the Hessian's appearing in the FTPL analysis beyond the case of independent, equal variance perturbations, to allow Gaussian processes with general covariance kernels.

While our argument is presented for $\f{BL}(X;\R)$, we believe it should also be possible to derive bounds for Sobolev spaces $W^{s,\infty}(X;\R)$, for any $s\in\N$, via a modified argument which involves also conditioning on derivatives.
We expect the argument to be less clean, and result in bounds that include $\E \sup_{x\in X} \norm{\grad \gamma}$ terms. 
We omit this extension for simplicity. 

We believe our work constitutes a first step towards technically-sound, yet implementable in practice, online learning algorithms for general state spaces---with numerics that resemble Thompson sampling implementations widely-used in Bayesian optimization.
We hope that future work can apply our approach in more general settings, beyond the bounded Lipschitz setting and full feedback regime, and ultimately lead to practical and effective algorithms for general non-discrete partially-observable problem classes.
To achieve this, what remains is to understand (a) what is the correct general analog of our key inequalities between the adversary's function class and the learner's prior, (b) using this, relate various adversary smoothness classes to corresponding Gaussian processes, or generalizations~thereof, and (c) determine relevant reductions from the partial feedback to the full feedback regime.

\section*{Acknowledgments}

This work is dedicated to the memory of David Draper, who supervised A.T.'s masters research, and believed in the value of the Bayesian approach to reasoning under uncertainty in order to make optimal decisions---including, and perhaps especially, in situations where the role of Bayesian thinking is subtle rather than obvious.
    
We thank James-Michael Leahy, Gergely Neu, Gabriele Farina, and Ziv Scully for sharing their thoughts about our approach at various points during the development of this work. 
A.T. would additionally like to thank Karthik Sridharan for making him aware of the work of \textcite{gravin2016towards}, which initiated the research which ultimately resulted in this work, and  Dan Roy, for introducing him to J.N. at ISBA 2024---a fitting venue, given the work's dedication---without which the collaboration would not have happened. 

A.T. was supported by Cornell University, jointly via the Center for Data Science for Enterprise and Society, the College of Engineering, and the Ann S. Bowers College of Computing and Information Science.
J.N. was supported by an NSERC Discovery Grant. 

\printbibliography

\newpage

\appendix

\section{General Regret Bounds for Thompson Sampling}
\label{apdx:technical}

In order to properly analyze prediction with expert advice with potentially-uncountably-many experts, we make sufficient assumptions for the respective zero-sum game to have well defined values, and ensure good-enough behavior with respect to differentiation in the sense of variational calculus. 
The following defines a sufficiently-general framework in which to define the action spaces of the learner and adversary, and to establish that the values of the game are well-defined.
Throughout this work, all topological vector spaces we consider are assumed Hausdorff.

\begin{definition}[General-state online learning game]
Let $X$ be second-countable compact Hausdorff.
Let $Y\subset\c{Y}\subseteq C(X;\R)$.
Let $\c{M}_s(X)$ be the space of signed finite measures on $X$, and let $\c{M}_1(X) \subseteq \c{M}_s(X)$ be the corresponding subset of probability measures on $X$.
Let $\pair{\.}{\.} : C(X;\R) \x \c{M}_s(X) \to \R$ be the pairing $\pair{\mu}{f} = \int_X f(x) \d\mu(x)$.
Endow $\c{M}_s(X)$ and $\c{M}_1(X)$ with the topology induced by pairing with continuous functions, which by compactness are automatically bounded: this is the topology of weak convergence of probability measures.
Define the \emph{general-state online learning game} to be following minimax game:
\1 At round $t$ the learner selects $p_t \in \c{M}_1(X)$, based on the history of the game.
\2 Then, the adversary selects $q_t \in \c{M}_1(Y)$, based on the history of the game and the learner's choice of $p_t$.
\3 Random samples $x_t,y_t$ are drawn from $p_t$, $q_t$, respectively in a manner so that $x_t$ and $y_t$ are conditionally independent of each other given the history.
\4 The zero-sum final values are given by regret, which is 
\[
R(p,q) &= \E_{\substack{x_t\~p_t\\y_t\~q_t}} \sup_{x\in X} \sum_{t=1}^T y_t(x) - \sum_{t=1}^T y_t(x_t)
\\
&= \E_{y_t\~q_t} \sup_{p\in\c{M}_1(X)} \sum_{t=1}^T \pair{y_t}{p} - \sum_{t=1}^T \pair{y_t}{p_t}
.
\]
\0 
\end{definition}

As a side note, these assumptions suffice to ensure that Sion's Minimax Theorem holds, though we will not need this for our results.
In this game, for a learner that plays Thompson sampling with prior distribution $q^{(\gamma)}$, a core component of our analysis is demonstrating that the learner's regret can be decomposed into the \emph{prior regret}  and the \emph{excess regret}.
For this, we define $\Gamma^*_t : C(X;\R) \to \R$ to be $\Gamma^*_t(f) = \E \sup_{x\in X} (f(x) + \gamma_{t:T}(x))$. 
Following \textcite[Chapter 30.5]{lattimore20}, the corresponding FTRL regularizer $\Gamma_t$ can be defined as the convex conjugate of $\Gamma^*_t$, namely $\Gamma_t = (\Gamma^*_t)^*$ in which case $\Gamma_t^* = (\Gamma_t)^*$, as we show in the sequel, which ensures there is no ambiguity.
To facilitate comparisons with other approaches, we mirror the notation of \textcite{lattimore20}, even though $\Gamma^*_t$ is the primitive definition in our~formalism.

\DefTS*

Recall again that this can equivalently be written 
\[
x_t &= \argmax_{x\in X} \sum_{\tau=1}^{t-1} y_\tau(x) + \sum_{\tau=t}^T \gamma_\tau(x)
&
(\gamma_t,..,\gamma_T) &\~ q^{(\gamma)} \given \gamma_1 = y_1, .., \gamma_{t-1} = y_{t-1}
.
\]
We start with the key regret decomposition.

\PropRegretDecomposition*

\begin{proof}
Start with 
\[
R(p,y) = \E \sup_{x\in X} \sum_{t=1}^T y_t(x) - \sum_{t=1}^T y_t(x_t) = \Gamma^*_{T+1}(y_{1:T}) - \E\sum_{t=1}^T \pair{y_t}{p_t}
.
\]
Next, consider the telescopic sum
\[
\sum_{t=1}^T \Gamma^*_{t+1}(y_{1:t}) - \Gamma^*_t(y_{1:t-1}) &= \sum_{t=2}^{T+1} \Gamma^*_t(y_{1:t-1}) - \sum_{t=1}^T \Gamma^*_t(y_{1:t-1}) 
\\
&= \Gamma^*_{T+1}(y_{1:T}) - \Gamma^*_1(0)
.
\]
Adding the difference of the two sides gives
\[
R(p,y) &= \Gamma^*_1(0) + \sum_{t=1}^T \Gamma^*_{t+1}(y_{1:t}) - \Gamma^*_t(y_{1:t-1}) - \E \pair{y_t}{p_t}
.
\]
Now, we add and subtract $\E\pair[0]{\gamma_t}{p_t^{(\gamma)}}$ from both sides to get 
\[
R(p,y) &= \ubr{\E\sbr{\Gamma^*_1(0) - \sum_{t=1}^T \pair[0]{\gamma_t}{p_t^{(\gamma)}}}}{R(p,q^{(\gamma)})} + \ubr{\sum_{t=1}^T \Gamma^*_{t+1}(y_{1:t}) - \Gamma^*_t(y_{1:t-1}) - \pair{y_t}{p_t} + \E\pair[0]{\gamma_t}{p^{(\gamma)}_t}}{E_{q^{(\gamma)}}(p,y)}
.
\]
The claim follows.
\end{proof}

The above claim generalizes to the case where, instead of the best-in-hindsight point, one instead works with some general comparator measure---we omit this as we will not need it.
We begin by verifying the basic structural properties of our setup.

\DefEqualizing*

\LemCentering*

\begin{proof}
This follows immediately by subtracting the respective means at each time from a given non-centered equalizing adversary.
\end{proof}

Next, we state conditions under which the prior regret is a lower bound on total regret.
The key assumption here is that of being supported on at most $Y$. 
While this fails in both of our examples for the Gaussian process priors we consider, by Gaussian universality we expect the prior-regret-based lower bound for a true equalizing and centered adversary given by a prior supported on $Y$ to be of the same order as the corresponding surrogate for the Gaussian process priors we used.

\PropLowerBound*

\begin{proof}
We have $R(p,q^{(\gamma)}) \leq \sup_q R(p,q)$: the claim follows by taking an infimum of both sides, and applying the property that $p\mapsto R(p,q^{(\gamma)})$ is constant in $p$.
\end{proof}

We will want to be able to bound the summands appearing in the excess regret, as we have defined it, using an appropriate Bregman divergence. 
However, since we are working in potentially-infinite-dimensional spaces, some care is required in order to ensure that Bregman divergences are well-defined and the corresponding convex analytic-tools we seek to apply are valid.
The first thing needed is to make sense of the derivatives appearing in the divergence.
Denote the extended real-line, as commonly occurs in convex analysis, by $\cl\R$.

\begin{definition}
\label{def:gateaux}
Let $G : \c{G} \to \cl\R$ be a proper convex function.
Define its \emph{Gâteaux derivative}, denoted $\p_{(\.)} G(\.) : \f{dom}(G) \x \c{G} \to \cl\R$, where the second argument refers to the subscript, by
\[
\p_v G(f) = \lim_{\eps\to0^+} \frac{G(f+\eps v) - G(f)}{\eps}
.
\]
\end{definition}

It is not hard to show that the Gâteaux derivative of a proper convex function always exists \cite[Theorem 2.1.13]{zalinescu02}.
However, the resulting Gâteaux derivative may fail to be finite-valued, continuous, or linear in its second argument.
All of these issues are already present in the case where the convex set is a closed interval: for example, linearity fails if one considers Gâteaux derivatives of the absolute value function at the origin, due to lack of smoothness.

We will show in the sequel that, so long as there are almost surely no ties---as is the case for our Gaussian process priors---the function $\Gamma^*_t$ of interest defines a bounded linear Gâteaux derivative.
Compare this with the situation that would occur for exponential-weights-type algorithms in our setting: in contrast with $\Gamma^*_t$, the usual Kullback--Leibler divergence $p \mapsto D_{\f{KL}}(p\from p_0)$, for a given $p_0\in\c{M}_1(X)$, is not smooth---for any given location, the set of directions in which its Gâteaux derivative is finite-valued is a strict subset of the space of signed measures and therefore would need to be characterized.

To proceed, we need an appropriate \emph{Envelope Theorem}: note that the usual results of \textcite{milgrom2002envelope} do not suffice, because they only provide inequalities, so we work with a sharp version similar to \textcite[Proposition 4.12]{bonnans13} and \textcite[Theorem 10.1]{basar08}.
Specifically, we work with the affine variant given by \textcite[Appendix B.6, Lemma 13]{xie24}: for completeness---and due to the presence of minor typographical errors in that work---we re-prove the result we need in \Cref{apdx:envelope}.
Before proceeding, we verify a certain maximizer choice continuity condition holds: in the following, let $\alpha_v(f) \in \argmax_{x\in X} f(x)$ denote a maximizer whose choice depends a function $v$. 

\begin{definition}
We call a function $\alpha:\c{Y}\to X$ with the property that $\alpha(f)\in \argmax_{x\in X} f(x)$ a \emph{tie-breaking rule}.
For a given tie breaking rule $\alpha:\c{Y}\to X$, and pair of functions $f\in\c{Y}$ and $v\in C(X; \R)$, define the \emph{variational value} $V^{(\alpha)}_{f,v}:\R \to \R$ by
\[
V^{(\alpha)}_{f,v}(s) = v(\alpha(f+sv)). 
\]
\end{definition}

\begin{lemma}
\label{lem:envelope_cont}
For every $v\in C(X; \R)$ there exists a tie-breaking rule  $\alpha_v:\c{Y}\to X$ such that, for any $y\in\c{Y}$, the variational value $V^{(\alpha_v)}_{y,v}$ is right-continuous with left limits.
In particular, for that tie-breaking rule, we have $v(\alpha_v(y+\eps v)) \to v(\alpha_v(y))$ as $\eps\to 0$.
\end{lemma}

\begin{proof}
Recall that, since $X$ is second-countable compact Hausdorff, it is metrizable, and in particular it is sequentially compact. Let $\rho$ be a metric that metrizes $X$.
In addition, note that the pair $(y,v)$ then defines an affine subset of $C(X;\R)$ given by
\[
\c{Y}_{y,v} = \cbr{y_s = y+sv \ :\  s\in \R}
.
\]
Let $\alpha_{\bdot}$ be any arbitrary tie-breaking rule. 
We will first show that $V_{y,v}^{(\alpha_{\bdot})}(s)$ is non-decreasing in $s$.
To do so, suppose that $s_1< s_2$.
Then, by the respective optimality, we have
\[
(y+s_1 v)(\alpha_{\bdot}(y_{s_2})) & \leq (y+s_1 v)(\alpha_{\bdot}(y_{s_1}))
&
(y+s_2 v)(\alpha_{\bdot}(y_{s_1})) & \leq (y+s_2 v)(\alpha_{\bdot}(y_{s_2})).
\]
Combining these we find that
\[
(y+s_1 v)(\alpha_{\bdot}(y_{s_1})) + (s_2-s_1) v(\alpha_{\bdot}(y_{s_1})) & \leq (y+s_1 v)(\alpha_{\bdot}(y_{s_2}))+ (s_2-s_1) v(\alpha_{\bdot}(y_{s_2})) 
\\
& \leq (y+s_1 v)(\alpha_{\bdot}(y_{s_1}))+ (s_2-s_1) v(\alpha_{\bdot}(y_{s_2})).
\]
We conclude
\[
v(\alpha_{\bdot}(y_{s_1})) \leq v(\alpha_{\bdot}(y_{s_2})).
\]
therefore $V_{y,v}^{(\alpha_{\bdot})}(s) = v(\alpha_{\bdot}(y_s))$ is non-decreasing in $s$ no matter which tie-breaking rule is used.

Using this, note that $V_{y,v}^{(\alpha_{\bdot})}:\R \to\R$, being non-decreasing, has at most countably many discontinuity points. 
Suppose that $s_c$ is a continuity point of $V_{y,v}^{(\alpha_{\bdot})}$.
If we change the tie-breaking rule at this point alone, while keeping all other points the same, by continuity the variational value must remain unchanged.
One can argue similarly along a partition of continuity points into dense subsets: we conclude all tie-breaking rules give the same variational value at continuity points, and that the set of points where $V_{y,v}^{(\alpha_{\bdot})}$ is discontinuous does not depend on the choice of tie-breaking rule $\alpha$---though different choices may yield different values between left and right limits.

We now show that a tie-breaking rule, $\alpha_v:\c{Y}\to X$, can be chosen so that $V_{y,v}^{(\alpha_v)}(s) = v(\alpha_v(y+sv))$ is right-continuous in $s$---and, hence right continuous with left limits, since it is non-decreasing---for all $y\in \c{Y}$. 
We will take our rule to be equivalent to $\alpha_{\bdot}$ if $s$ is a continuity point of $V_{y,v}^{(\alpha_{\bdot})}(s)$. 
We now need only define the rule for discontinuity points of $V_{y,v}^{(\alpha_{\bdot})}$.

Let $s_d$ be a discontinuity point of $V_{y,v}^{(\alpha_{\bdot})}$.
Since $X$ is sequentially compact, every sequence in $X$ has a convergent subsequence: thus, there is a decreasing sequence of continuity points $s_k$, $k\in\N$, of $V_{y,v}^{(\alpha_{\bdot})}$ with $s_k\geq s_d$ and $s_k\to s_d$ for which $x_k=\alpha_{\bdot}(y_{s_k})$ converges in $X$. 
Let $x_\star$---depending on $y$, $v$, and the sequence $s_k$---be the limit of an arbitrarily chosen such converging sequence, and let $\eps>0$. 
Since $y_{s_d}$ must be continuous at $x_\star$, and since $v$ is bounded, there is a $\delta>0$ and an $M>0$ so that for $\rho(x,x_\star)<\delta$, we have
\[
\abs{y_{s_d}(x) - y_{s_d}(x_\star)}& <\eps 
&
\norm{v}_\infty &\leq M.
\]
Since $x_k\to x_\star$, there is a $k_0\in\N$ with $\rho(x_k,x_\star) \leq \delta$ for all $k\geq k_0$.
Thus, for $k\geq k_0$, we get
\[
y_{s_d}(x_\star) \geq y_{s_d}(x_k)-\eps \geq y_{s_d}(x_k)+(s_k-s_d)v(x_k) -\eps- (s_k-s_d) \norm{v}_\infty
\] 
and 
\[
y_{s_k}(x_k) -\eps- (s_k-s_d) \norm{v}_\infty & \geq y_{s_k}(\alpha_{\bdot}(y_{s_d})) -\eps- (s_k-s_d) \norm{v}_\infty
\\
&\geq y(\alpha_{\bdot}(y_{s_d})) -\eps- 2  (s_k-s_d) \norm{v}_\infty.
\]
Since $\eps$ is arbitrary, and for a given $\eps$ we can take $k$ to be arbitrarily large, and hence $s_k-s_d$ arbitrarily small, it follows that $x_\star \in A(y)$.
Define the tie breaking rule $\alpha_v$ at the point $y_{s_d}$ by $\alpha_v(y_{s_d}) = x_\star$. 
Since $v$ is continuous and $x_k\to x_\star$ it follows that the variational value $V_{y,v}^{(\alpha_v)}(s) = v(\alpha_v(y+sv))$ is right-continuous in $s$ at $s_d$. 

We can modify $\alpha_{\bdot}$ at all discontinuity points of $V_{y,v}^{(\alpha_{\bdot})}$ at once in a similar way, yielding a tie-breaking rule $\alpha_v$ such that $V_{y,v}^{(\alpha_v)}$ is right-continuous everywhere on the affine space $\c{Y}_{y,v}$, and these rules can be chosen to be consistent between different function pairs $(y,v)$ that give rise to the same affine subspace of $\c{Y}$. 
We can then therefore join these rules together for different affine subspaces of $\c{Y}_{y,v}\subseteq\c{Y}$ to get a single tie-breaking rule $\alpha_v$ with $V_{y,v}^{(\alpha_v)}$ right-continuous with left limits for all $y\in\c{Y}$, as was claimed.
\end{proof}

\begin{lemma}
\label{lem:derivative}
Assume that $q^{(\gamma)}$ is independent across time, and is chosen such that maximizers defining $\Gamma^*_t$ are almost surely unique in the sense that, for all $f\in C(X;\R)$ and all $t\in[T]$, we have $q^{(\gamma)} (\{\gamma : |\argmax_{x\in X}(f(x)+\gamma_{t:T}(x))\}=1|)=1$.
Then $\Gamma^*_t$ is Lipschitz with respect to the supremum norm on $C(X;\R)$, and its Gâteaux derivative is
\[
\p_v \Gamma^*_t(f) = \pair[1]{v}{p^{(f)}_t}
\]
where $p_t^{(f)}$ is defined as the distribution of $\displaystyle\argmax_{x\in X} (f(x) + \gamma_{t:T}(x))$.
\end{lemma}

\begin{proof}
We first prove the second assertion.
Note first that, since $X$ is compact, $\c{Y}$ consists entirely of continuous functions, and $\gamma_{t:T}$ is almost surely continuous by assumption, the supremum $\sup_{x\in X} f(x) + \gamma_{t:T}(x)$ is almost surely achieved.
Mirroring the notation of \Cref{lem:envelope_sharp}, define $\c{V}^{(g)}(f) = \sup_{x\in X} f(x) + g(x)$, and $\c{L}^{(g)}(x,f) = f(x) + g(x)$, for which we have $\p_v \c{L}^{(g)}(x,f) = v(x)$, with differentiation taken in the second argument.
Write
\[
\p_v \Gamma^*_t(f) &= \p_v \E_{\gamma_{t:T}\~q^{(\gamma)}} \c{V}^{(\gamma_{t:T})}(f) 
\\
\overset{\t{(i)}}&= \E_{\gamma_{t:T}\~q^{(\gamma)}} \p_v \c{V}^{(\gamma_{t:T})}(f)
\\
\overset{\t{(ii)}}&= \E_{\gamma_{t:T}\~q^{(\gamma)}} \max_{x^* \in A(f + \gamma_{t:T})} v(x^*) 
\\
\overset{\t{(iii)}}&= \E_{x\~p_t^{(f)}} v(x)
\]
where (i) holds by a Dominated Convergence Theorem variant given by \textcite[Theorem 2.13]{legall2022measure} where all required conditions follow by convexity of $\Gamma^*_t$, and (ii) holds by \Cref{lem:envelope_sharp,lem:envelope_cont}, and (iii) holds because $p^{(f)}_t$ is the pushforward of $q^{(\gamma)}$ by $\gamma\mapsto\argmax_{x\in X}(f(x)+\gamma(x))$, using the assumption that maximizers are almost surely unique.
Since 
\[
\sup_{\norm{v}_\infty\leq 1} \pair[1]{v}{p^{(f)}_t} \leq 1
\]
the Lipschitz property, and the claim, follows.
\end{proof}

The next step is to establish a form of Fenchel duality induced by $\Gamma^*_t$.
To focus attention on regret-theoretic aspects, and given the generality of our setting, we adopt a minimalist approach which consists of deriving just-enough consequences of this duality for our calculation to go through.
To do so, let $C(X;\R)^*$ be the topological dual of $C(X;\R)$, and denote the canonical pairing by $\pair{\.}{\.} : C(X;\R) \x C(X;\R)^* \to \R$.
The resulting notational overlap is unambiguous: by the Reisz--Markov--Kakutani Representation Theorem \cite[Theorem 6.19]{rudin87},\footnote{Note that, while this reference proves the complex measure variant of this result, its real analog follows by the same argument. Also, recall that since we have assumed $X$ second-countable compact Hausdorff, the Borel regularity property in the aforementioned result holds for all of $\c{M}_s(X)$ by standard theory.} the dual $C(X;\R)^*$ in bijective isometry with $\c{M}_s(X)$, viewed as a Banach space under the total variation norm.

\begin{lemma}
\label{lem:biconjugate}
Define the Fenchel conjugate $\Gamma_t^{**} : C(X;\R)^* \to \cl\R$,  by
\[
\Gamma_t^{**}(\ell) = \sup_{f\in C(X;\R)} \pair[1]{f}{\ell} - \Gamma^*_t(f) 
\]
with the convention that signed measures are identified with the bounded linear functionals they give rise to by way of integration.
Then for $f\in C(X;\R)$ we have
\[
\Gamma_t^{**}(p^{(f)}_t) &= \pair[1]{f}{p^{(f)}_t} - \Gamma^*_t(f)
.
\]
\end{lemma}

\begin{proof}
\textcite[Theorem 2.4.2(iii)]{zalinescu02}: to apply this, we verify our our setting satisfies the assumptions assumed therein---which it does, since the space of continuous functions $C(X;\R)$ is Hausdorff and locally convex.
\end{proof}

\begin{lemma}
\label{lem:bregman_duality}
For $f,f'\in C(X;\R)$ and $\ell\in C(X;\R)^*$, define the \emph{Bregman divergences}
\[
D_{\Gamma^*_t}(f \from f') &= \Gamma^*_t(f) - \Gamma^*_t(f') - \pair[1]{f-f'}{p^{(f')}_t}
\\
D_{\Gamma^{**}_t}(\ell \from p^{(f)}_t) &= \Gamma^{**}_t(\ell) - \Gamma^{**}_t(p^{(f)}_t) - \pair[1]{f}{\ell - p^{(f)}_t}
.
\]
Then we have the Bregman duality formula
\[
D_{\Gamma^*_t}(f \from f') = D_{\Gamma_t^{**}}(p_t^{(f')} \from p_t^{(f)})
.
\]
\end{lemma}

\begin{proof}
This follows immediately from \Cref{lem:biconjugate} and the respective definitions.
\end{proof}

\begin{lemma}
\label{lem:bregman_conjugate}
We have
\[
D_{\Gamma^{**}_t}(\. \from p^{(f)}_t)^*(f') = \Gamma^*_t(f + f') + \Gamma^{**}_t(p^{(f)}_t) - \pair[1]{f}{p^{(f)}_t}
.
\]
\end{lemma}

\begin{proof}
Using the definition of a convex conjugate and linearity, write
\[
D_{\Gamma_t^{**}}(\. \from p^{(f)}_t)^*(f') &= \sup_{\ell\in C(X;\R)^*} \pair{f'}{\ell} - D_{\Gamma^{**}_t}(\ell \from p^{(f)}_t)
\\
&= \del{\sup_{\ell\in C(X;\R)^*} \pair{f + f'}{\ell} - \Gamma^{**}_t(\ell)}  + \Gamma^{**}_t(p^{(f)}_t) - \pair[1]{f}{p^{(f)}_t}
\\
&= \Gamma^*_t(f + f') + \Gamma^{**}_t(p^{(f)}_t) - \pair[1]{f}{p^{(f)}_t}
\]
where the final step follows by \textcite[Theorem 2.3.3(iii)]{zalinescu02}---where, to apply this result, lower semi-continuity of $\Gamma^*_t$ follows from the Lipschitz property established in \Cref{lem:derivative}.
\end{proof}

Note that, given our choice to mirror the notation of \textcite[Chapter 30.5]{lattimore20}, the definition of the FTRL regularizer is $\Gamma_t = \Gamma^{**}_t$: in what follows, we write the latter expression rather than the former for notational precision---recall that, in our formalism, $\Gamma^*_t$ is the primitive definition.
A point of technical subtlety is our definition of the Bregman divergence $D_{\Gamma^{**}_t}(\.\from\.)$: this definition essentially \emph{posits} what the respective derivative term appearing in it should be.
One can show, for virtual adversaries with large-enough support over the space of continuous functions, that this definition coincides with the standard one.
In cases where this fails, the respective Gâteaux derivative may not be linear: our definition then corresponds to picking an appropriate element from the respective subdifferential.
Using these calculations, we are finally ready to prove our Bregman-divergence-based bound on the excess regret.

\PropBregmanBound*

\begin{proof}
By definition
\[
E_{q^{(\gamma)}}(p,y) = \sum_{t=1}^T \ubr{\Gamma^*_{t+1}(y_{1:t}) - \Gamma^*_t(y_{1:t-1}) - \pair{y_t}{p_t} + \E\pair[0]{\gamma_t}{p_t^{(\gamma)}}}{E_t}
.
\]
We bound this term-by-term.
Recall that, 
Using this, write
\begingroup
\allowdisplaybreaks
\[
E_t &= \Gamma^*_{t+1}(y_{1:t}) - \Gamma^*_t(y_{1:t-1}) - \pair{y_t}{p_t} + \E\pair[0]{\gamma_t}{p_t^{(\gamma)}}
\\
\overset{\t{(i)}}&= \Gamma^*_{t+1}(y_{1:t}) - \Gamma^*_t(y_{1:t-1}) - \E\pair{y_t - \gamma_t}{p_t}
\\
&= \E_{\gamma_t\~q_t^{(\gamma)}} \pair{y_t - \gamma_t}{p_{t+1} - p_t} - \Gamma^*_{t+1}(y_{1:t-1} + \gamma_t) + \Gamma^*_{t+1}(y_{1:t}) + \pair{\gamma_t - y_t}{p_{t+1}}
\\
&= \E_{\gamma_t\~q_t^{(\gamma)}} \pair{y_t - \gamma_t}{p_{t+1} - p_t} - D_{\Gamma^*_{t+1}}(y_{1:t-1} + \gamma_t \from y_{1:t})
\\
\overset{\t{(ii)}}&= \pair{y_t}{p_{t+1} - p_t} - \E_{\gamma_t\~q_t^{(\gamma)}}  D_{\Gamma_{t+1}^{**}}(p_{t+1} \from p_{y_{1:t-1} + \gamma_t})
\\
\overset{\t{(iii)}}&= \E_{\gamma_t\~q_t^{(\gamma)}} \pair{y_t}{p_{t+1}} - \pair{y_t}{p_{y_{1:t-1} + \gamma_t}} - D_{\Gamma_{t+1}^{**}}(p_{t+1} \from p_{y_{1:t-1} + \gamma_t})
\\
\overset{\t{(iv)}}&\leq \E_{\gamma_t\~q_t^{(\gamma)}}D_{\Gamma_{t+1}^{**}}(\. \from p_{y_{1:t-1} + \gamma_t})^*(y_t) + D_{\Gamma_{t+1}^{**}}(p_{t+1} \from p_{y_{1:t-1} + \gamma_t})
\\
&\qquad- \pair{y_t}{p_{y_{1:t-1} + \gamma_t}} - D_{\Gamma_{t+1}^{**}}(p_{t+1} \from p_{y_{1:t-1} + \gamma_t})
\\
&= \E_{\gamma_t\~q_t^{(\gamma)}}D_{\Gamma_{t+1}^{**}}(\. \from p_{y_{1:t-1} + \gamma_t})^*(y_t) - \pair{y_t}{p_{y_{1:t-1} + \gamma_t}}
\\
\overset{\t{(v)}}&= \E_{\gamma_t\~q_t^{(\gamma)}} \Gamma^*_{t+1}(y_{1:t-1} + \gamma_t + y_t) - \pair{y_t}{p_{y_{1:t-1} + \gamma_t}}
\\
&\qquad+ \Gamma^{**}_{t+1}(p_{y_{1:t-1} + \gamma_t}) - \pair{y_{1:t-1} + \gamma_t}{p_{y_{1:t-1} + \gamma_t}} 
\\
\overset{\t{(vi)}}&= \E_{\gamma_t\~q_t^{(\gamma)}} \Gamma^*_{t+1}(y_{1:t-1} + \gamma_t + y_t) - \Gamma^*_{t+1}(y_{1:t-1} + \gamma_t) - \pair{y_t}{p_{y_{1:t-1} + \gamma_t}}
\\
&= \Gamma^*_t(y_{1:t}) - \Gamma^*_t(y_{1:t-1}) - \pair{y_t}{p_t}
\\
&= D_{\Gamma^*_t}(y_{1:t} \from y_{1:t-1})
.
\]
\endgroup
where (i) follow by equalization, since $\E\pair{\gamma_t}{p} = \E\pair{\gamma_t}{p'}$ for all $p$ and $p'$, (ii) follows by a combination of equalization and \Cref{lem:bregman_duality}, (iii) follows by the Tower Rule, (iv) follows by applying Young's inequality with respect to the convex function $\ell \mapsto D_{\Gamma_{t+1}^{**}}(\ell \from p_{y_{1:t-1} + \gamma_t})$, (v) follows by \Cref{lem:bregman_conjugate}, (vi) follows by \Cref{lem:biconjugate}, and all remaining lines follow by definition.
\end{proof}

\subsection{A Sharp Envelope Theorem}
\label{apdx:envelope}

We now prove the necessary Envelope Theorem used in \Cref{lem:derivative}.

\begin{lemma}
\label{lem:envelope}
Let $\Theta$ be a subset of a topological vector space, and let $X$ be an arbitrary set.
Let $\c{L} : X \x \Theta \to \R$ be bounded above, and define $\c{V}$ to be
\[
\c{V}(\theta) = \sup_{x\in X} \c{L}(x,\theta)
\]
Suppose that, for every $\theta\in\Theta$, the supremum is achieved, and let $\c{A}(\theta)$ be the maximizer set.
For any $v$, suppose that the Gâteaux derivatives $\p_v \c{V}(\theta)$ and $\p_v \c{L}(x,\theta)$ exist and are finite-valued, with the convention that the Gâteaux derivative of $\c{L}$ is taken in its second argument.
Then
\[
\sup_{x^*\in \c{A}(\theta)} \p_v \c{L}(x^*, \theta) \leq \p_v \c{V}(\theta)
.
\]
\end{lemma}

\begin{proof}
For a given $\theta\in\Theta$, let $x^*(\theta) \in \c{A}(\theta)$ be an arbitrary maximizer. 
Observe that, for all $\theta'\in\Theta$ and all $x^*(\theta')\in A(\theta')$, we have
\[
\c{L}(x^*(\theta),\theta') \leq \c{L}(x^*(\theta'),\theta') = \c{V}(\theta')
\]
by optimality, with equality if $\theta' = \theta$.
Now, let $\eps>0$ be sufficiently small, and choose $\theta' = \theta + \eps v$.
Then, if we subtract $\c{L}(x^*(\theta), \theta) = \c{V}(\theta)$ from both sides and form difference quotients, we get
\[
\frac{\c{L}(x^*(\theta),\theta + \eps v) - \c{L}(x^*(\theta), \theta)}{\eps} \leq \frac{\c{V}(\theta + \eps v) - \c{V}(\theta)}{\eps}
.
\]
Since we have assumed both derivatives to exist, taking limits gives
\[
\p_v \c{L}(x^*(\theta),\theta) \leq \p_v \c{V}(\theta)
\]
for all choices $x^*(\theta)\in \c{A}(\theta)$.
\end{proof}

\begin{lemma}
\label{lem:envelope_sharp}
With the assumptions and notation of \Cref{lem:envelope}, suppose further that:
\1 For all $x$, $\c{L}(x,\theta)$ is affine in $\theta$, namely $\c{L}(x,\theta) = \c{L}_0(x,\theta) + m(x)$ where $\c{L}_0$ is linear in $\theta$.
\2 For any $\theta$ and any $v$, there exists $x^*_v(\theta) \in \c{A}(\theta)$ such that, for all sufficiently-small $\eps>0$, there exist $x^*_v(\theta + \eps v) \in A(\theta + \eps v)$, such that we have $\c{L}_0(x^*_v(\theta+\eps v),v) \to \c{L}_0(x^*_v(\theta),v)$.
\0 
Then the supremum over $\c{A}(\theta)$ in \Cref{lem:envelope} is achieved---not necessarily uniquely---and we have
\[
\max_{x^*\in \c{A}(\theta)} \p_v \c{L}(x^*, \theta) = \p_v \c{L}(x^*_v(\theta),\theta) = \p_v \c{V}(\theta)
.
\]
\end{lemma}

\begin{proof}
Note first, by optimality, that $\c{L}(x^*(\theta + \eps v),\theta) \leq \c{L}(x^*(\theta),\theta)$.
Using this, for any $\eps>0$,~write
\[
\frac{\c{V}(\theta + \eps v) - \c{V}(\theta)}{\eps} &= \frac{\c{L}(x^*(\theta + \eps v),\theta + \eps v) - \c{L}(x^*(\theta),\theta)}{\eps}
\\
&= \frac{\c{L}(x^*(\theta + \eps v),\theta + \eps v) - \c{L}(x^*(\theta + \eps v),\theta)}{\eps}
\\
&\quad+ \frac{\c{L}(x^*(\theta + \eps v),\theta) - \c{L}(x^*(\theta),\theta)}{\eps}
\\
\overset{\t{(i)}}&\leq \frac{\c{L}(x^*(\theta + \eps v),\theta + \eps v) - \c{L}(x^*(\theta + \eps v),\theta)}{\eps}
\\
\overset{\t{(ii)}}&= \c{L}_0(x^*(\theta + \eps v),v)
\]
where (i) follows by the aforementioned optimality inequality, and (ii) follows because $\c{L}$ was assumed affine.
Next, we choose $x^*(\theta) = x^*_v(\theta)$ and $x^*(\theta+\eps v) = x^*_v(\theta+\eps v)$ according to the claim's assumption, and use this assumption to take limits of both sides: this gives
\[
\p_v \c{V}(\theta) \leq \c{L}_0(x^*_v(\theta), v) = \p_v \c{L}(x^*_v(\theta), \theta)
\]
where the final equality is the expression for the Gâteaux derivative of an affine function.
Combining this expression with the \Cref{lem:envelope}, we conclude that the respective inequality is tight, and that each $x^*_v(\theta)$ achieves the supremum over $\c{A}(\theta)$.
The claim follows.
\end{proof}

\section{Minimax Regret Bounds for Gaussian Virtual Adversaries}
\label{apdx:gaussian}

We now study the behavior of Gaussian virtual adversaries, beginning with their prior regret.

\LemPrior*

\begin{proof}
Immediate.
\end{proof}

Since this quantity is necessarily setting-dependent, we proceed to analyze excess regret.
For this, we need a result which lets us move from the Bregman divergence on the infinite-dimensional space to a Bregman divergence on a finite-dimensional space obtained from a cover.
Using this, we will bound the Bregman divergence on the cover by quantities which are independent of cover size, and then take a limit to remove the extra term arising from the cover at the very end.
The reason we adopt this approach---rather than working with the infinite-dimensional form directly---is because it will significantly streamline the key argument, allowing us to work with simple matrices and vectors, and avoid the need to condition on probability-zero events or handle other analytic subtleties.

We say that a function $f$ has a global modulus of continuity if there is a continuous function $\omega : \cl\R^+ \to \cl\R^+$ with $\omega(0) = 0$ which satisfies
\[
\sup_{\norm{x-x'}<h}\abs{f(x)-f(x')}\leq \omega(h)
.
\]
and say that a stochastic process has a global expected modulus of continuity if this expression holds in expectation.
Note that, since $X$ is compact, all continuous functions admit a global modulus of continuity.

\begin{lemma}
\label{lem:cover}
Let $X_h \subseteq X$ be a cover with radius $h$, and let $\psi : \cl\R_+ \to \cl\R_+$ be the expected global modulus of continuity of $\gamma\~[GP](0,k)$.
Then, letting $\omega_t$ be the global modulus of continuity of $y_{1:t}$, we have
\[
D_{\Gamma^*_t}(y_{1:t} \from y_{1:t-1}) & \leq  D_{\Gamma_{h,t}^*}(y_{h,1:t} \from y_{h,1:t-1}) + \c{O}(\psi(h)) + \c{O}(\omega_t(h))
\]
where $y_{h,1:t} = y_{1:t}(X_h)$, and $\displaystyle\Gamma^*_{h,t}(f) = \E \max_{x\in X_h} \del{f(x) + \gamma_{t:T}(x)}$.
\end{lemma}

\begin{proof}
Recall that our Bregman divergence can be written in probabilistic form as
\[
D_{\Gamma^*_t}(y_{1:t} \from y_{1:t-1}) = \E ((y_{1:t} + \gamma_{t:T})(x^*_{y_{1:t} + \gamma_{t:T}}) - (y_{1:t} + \gamma_{t:T})(x^*_{y_{1:t-1} + \gamma_{t:T}}))
\]
where we pass to the cover in stages, separately for the deterministic and random part of the sum.
To ease notation, let $\Upsilon = y_{1:t} + \gamma_{t:T}$.
For the deterministic part let its global modulus of continuity be denoted by $\omega_t$: this always exists, since $y_t$ is assumed continuous and $X$ is compact.
With this, we have
\[
y_{1:t}(x^*_\Upsilon) -  y_{1:t}(x^*_{\Upsilon - y_t}) &= y_{1:t}(x^*_\Upsilon) -  y_{1:t}(x^*_\Upsilon|_{X_h}) + y_{1:t}(x^*_\Upsilon|_{X_h})
\\
&\quad- y_{1:t}(x^*_{\Upsilon - y_t}) + y_{1:t}(x^*_{\Upsilon - y_t}|_{X_h}) -  y_{1:t}(x^*_{\Upsilon - y_t}|_{X_h})
\\
&\leq  y_{1:t}(x^*_\Upsilon|_{X_h}) -  y_{1:t}(x^*_{\Upsilon - y_t}|_{X_h}) + 2\omega_t(h)
\]
where the notation $x^*_\Upsilon|_{X_h}$ refers to the maximizer of $\Upsilon$ over $X_h$, and similarly for $x^*_{\Upsilon-y_t}$.
For the stochastic part, we have
\[
\E(\gamma_{t:T}(x^*_{y_{1:t} + \gamma_{t:T}}) - \gamma_{t:T}(x^*_{y_{1:t-1} + \gamma_{t:T}})) &\leq \gamma_{t:T}(x^*_\Upsilon|_{X_h}) -  \gamma_{t:T}(x^*_{\Upsilon - y_t}|_{X_h}) 
\\
&\quad+ 2\sqrt{T-t+1} \psi(h)
\]
where $\psi(h)$ is the expected global modulus of continuity for $\gamma$ at $h$.
\end{proof}

We will need a certain property of truncated multivariate normals.
Let the notation $z\~\f{TN}(\mu, \Sigma; \ell, u)$ refer to a truncated multivariate with mean $\mu$, covariance kernel $k$, lower truncation level $\ell$, and upper truncation level $u$.
To ensure well-definedness, we assume $\ell < \mu \leq u$.

\begin{lemma}
\label{lem:truncated_normal}
Let $z\~[TN](\mu,\Sigma;-\infty,\alpha)$ where $\alpha\in\R^d$, and $\Sigma$ is strictly positive definite.
Then 
\[
\E(z) = \mu - \Sigma (p_{z_i}(\alpha_i))_{i=1}^n
\]
where $p_{z_i}$ is the marginal probability density of $z_i$.
\end{lemma}

\begin{proof}
\textcite[Chapter 45]{kotz00}.
\end{proof}

We are now ready to prove the key Hessian bound.

\ThmHessianBound*

\begin{proof}
We first pass to the cover, using \Cref{lem:cover}.
There, by Taylor's Theorem, we have 
\[
D_{\Gamma^*_{h,t}}(y_{h,1:t} \from y_{h,1:t-1}) &= \frac{1}{2} \int_0^1 y_{h,t}^T \grad^2\Gamma^*_{h,t}(y_{1:t-1} + \theta y_t) y_{h,t} \d\theta 
\\
&= \frac{1}{2\sqrt{T-t+1}} \int_0^1 \E y_{h,t}^T \1_{i^*_{y_{1:t-1} + \theta y_t +\gamma_{t:T}}} \gamma_h^T K_h^{-1} y_{h,t} \d\theta
\]
where $(K_h)_{ij} = k(x_i, x_j)$ is the kernel matrix, $\gamma_h\~[N](0,K_h)$, $\1_j$ refers to a vector which is equal to one at index $j$ and zero otherwise, and the explicit formula, following \textcite{abernethy2016perturbation}, can be derived using Gaussian integration by parts.
To ease notation, we let $f = y_{1:t-1} + \theta y_t +\gamma_{t:T}$, and omit $h$-subscripts and time-subscripts henceforth. 
We first apply the Tower Rule to condition the expression on $i^*_{f+\gamma} = i$ and $(f+\gamma)_i = \alpha$ to obtain
\[
\E y^T \1_{f+\gamma} \gamma^T K^{-1} y &= \E y_i \E\del{\gamma^T K^{-1} y \given i^*_{f+\gamma} = i, (f+\gamma)_i = \alpha}
\\
&= \E y_i \E\del{f + \gamma \given i^*_{f+\gamma} = i, (f+\gamma)_i = \alpha}^T K^{-1} y - f^T K^{-1} y
\]
where the outer expectation is taken over $i$, $\alpha$.
Given the index $i$ and value $\alpha$, the random vector $f+\gamma$ is a truncated multivariate Gaussian, specifically
\[
f + \gamma \~[TN](f + K_{(\.)i} K_{ii}^{-1} (\alpha - f_i), K - K_{(\.)i} K_{ii}^{-1} K_{i(\.')}; -\infty, \alpha)
.
\]
We now take the inner expectation: by a combination of \Cref{lem:truncated_normal} for $j\neq i$, and being equal to $\alpha$ for $j=i$, the expectation is equal to
\[
\E\del{f+\gamma  \given i^*_{f+\gamma} = i, (f+\gamma)_i = \alpha} &= f + K_{(\.)i} K_{ii}^{-1} \gamma_i
\\
&\quad- \del{K_{(\.,\.')} - K_{(\.)i} K_{ii}^{-1} K_{i(\.')}}p_{(f+\gamma)(\.')}(\alpha)
\]
where $p_{(f+\gamma)(x)}(\.)$ is the marginal density of the truncated normal, and the vector $p_{(f+\gamma)(\.')}(\alpha)$ is defined at all points except $i$ and therefore has dimension $|X_h| - 1$.
Together, this gives 
\[
&\E\del{\gamma \given i^*_{f+\gamma} = i, (f+\gamma)_i = \alpha}^T K^{-1} y
\\
&\quad= y^T K^{-1} K_{(\.)i} K_{ii}^{-1} \gamma_i - y^T (K^{-1} K_{(\.,\.')} - K^{-1} K_{(\.)i} K_{ii}^{-1} K_{i(\.')}) p_{(f+\gamma)(\.')}(\alpha)
\\
&\quad= y_i K_{ii}^{-1} \gamma_i - y^T (I_{-i} - K^{-1} K_{(\.)i} K_{ii}^{-1} K_{i(\.')}) p_{(f+\gamma)(\.')}(\alpha)
\\
&\quad= y_i K_{ii}^{-1} \gamma_i - \del{y_{-i} - K_{(\.)i} K_{ii}^{-1} y_i}^T p_{(f+\gamma)(\.)}(\alpha)
\]
where we have used the interpolation formula $y^T K^{-1} K_{(\.)i} = y_i$ for both terms, $I_{-i}$ is a rectangular matrix with $|X_h|$ rows and $|X_h| - 1$ columns, which is equal to the identity matrix but with the $i$th column removed, and the subscript of $y_{-i}$ refers to all indices except $i$.
Before proceeding further, we prove a certain identity involving terms which arise from the Hessian at hand.
Letting $1$ be the constant vector, note by symmetry of the Hessian, without any conditioning, that for all $w$ we have
\[
\E w^T \1_{i^*_{f+\gamma}} \gamma^T K^{-1} 1 = \E \gamma^T K^{-1} w = 0
.
\]
Conditioning on $i^*_{f+\gamma} = i$, we obtain 
\[
\E w_i \E\del{\gamma^T K^{-1} 1\given i^*_{f+\gamma}=i} = 0
\]
which, since this holds for all $w$, implies
\[
\E\del{\gamma^T K^{-1} 1 \given i^*_{f+\gamma}=i} = \E\del{\gamma \given i^*_{f+\gamma}=i}^T K^{-1} 1 = 0
.
\]
Plugging the form calculated above into this expression gives the identity
\[ 
\E \del{1 - K_{(\.)i} K_{ii}^{-1}}^T p_{(f+\gamma)(\.)}(\alpha) = \E K_{ii}^{-1} \gamma_i
.
\]
We now start with the bounds: write
\[
&\E y^T \1_{i^*_{f+\gamma}} \gamma^T K^{-1} y = \E y_i \E\del{\gamma^T K^{-1} y \given i^*_{f+\gamma} = i, (f+\gamma)_i = \alpha}
\\
&\quad= \E y_i^2 K_{ii}^{-1} \gamma_i - y_i \del{y_{-i} - K_{(\.)i} K_{ii}^{-1} y_i}^T p_{(f+\gamma)(\.)}(\alpha)
\\
&\quad= \E \frac{y(x_i)^2 \gamma(x_i)}{k(x_i,x_i)} - y(x_i) \sum_{j\neq i} \del{y(x_j) - \frac{k(x_i,x_j)}{k(x_i,x_i)} y(x_i)} p_{(f+\gamma)(x_j)}(\alpha)
\\
&\quad\leq \E \frac{y(x_i)^2 \gamma(x_i)}{k(x_i,x_i)} + C_{Y,k} |y(x_i)| \E \sum_{j\neq i}\del{1 - \frac{k(x_i,x_j)}{k(x_i,x_i)}} p_{(f+\gamma)(x_j)}(\alpha)
\\
&\quad= \E (y(x_i)^2 + C_{Y,k} |y(x_i)|)\frac{\gamma(x_i)}{k(x_i,x_i)}
\]
where the outer expectations are over $i$ and inner expectations are over $\alpha$, the inequality follows by assumption, and final equality follows by the derived identity together with positivity of $p_{(f+\gamma)(\.)}(\alpha)$.
Collecting terms, we obtain
\[
D_{\Gamma^*_t}(y_{1:t} \from y_{1:t-1}) &\leq \frac{1}{2\sqrt{T-t+1}} \int_0^1 \E \frac{(y(x_i)^2 + C_{Y,k} |y(x_i)|)\gamma(x_i)}{k(x_i,x_i)} \d\theta 
\\
&\quad+ \c{O}\del{\sqrt{h\log\frac{1}{h}}} + \c{O}(\omega_t(h))
\]
and the claim follows by taking $h\to 0$.
\end{proof}

We are now ready to prove the main regret bound.

\ThmTS*

\begin{proof}
We apply the developed theory: decomposing total regret into prior regret and excess regret, we can bound the prior regret using \Cref{lem:matern_sup}, which gives the first term in the bound.
We then bound the excess regret by the respective Bregman divergences, and bound those term-by-term using \Cref{thm:bregman_hessian_bound}.
This gives
\[
R(p,q) &\leq R(p, q^{(\gamma)}) + \E\sum_{t=1}^T \frac{\beta(\beta + C_{Y,k})}{2\sigma^2\sqrt{T-t+1}} \E\sup_{x\in X} \gamma(x)
\\
&\leq \sqrt{T}\del{1 + \frac{\beta(\beta + C_{Y,k})}{\sigma^2}} \E\sup_{x\in X} \gamma(x)
\]
using $\sum_{t=1}^T \frac{1}{\sqrt{T-t+1}} = \sum_{t=1}^T \frac{1}{\sqrt{t}} \leq 2\sqrt{T}$.
\end{proof}

This gives the general theory. 
We now study the behavior of $C_{Y,k}$ for the spaces of interest.

\section{Regret Bounds in General Spaces}
\label{apdx:rates}

We now apply \Cref{thm:ts} to a number of examples in various concrete settings.
The first of these is the classical finite setting, where we verify that we can recover the minimax rate.

\CorDiscreteRegret*

\begin{proof}
In this situation, $k(x,x') = 0$ for $x\neq x'$, so the key inequality reduces to 
\[
\sup_{\norm{y}_\infty\leq1} y(x) \leq C_{Y,k}
\]
which holds with $C_{Y,k} = 1$.
Applying \Cref{thm:ts} gives 
\[
R(p,q) \leq \del{\sigma + \frac{2}{\sigma}} \sqrt{2T\log N}
\]
and the claim follows by taking $\sigma = \sqrt{2}$.
\end{proof}

This should be compared with what would have happened for a constant-learning-rate-based FTPL algorithm, whose regularizer we denote by $\Gamma_{\eta,t}$.
There, repeating the same argument, we would have obtained $\Gamma^*_{\eta,t}(0) \leq \eta\sqrt{2\log N}$ for the first term in our bound, and $\frac{T}{\eta}\sqrt{2\log N}$ for the second term.
Balancing terms using a learning rate of $\sqrt{T}$, this would have given
\[ 
R(p,q) &\leq  \del{\eta + \frac{T}{\eta}} \sqrt{2 \log N} = 2\sqrt{2T\log N}
\]
which matches our rates, but improves the constant by a factor of $\sqrt{2}$.
Our analysis, while motivated by the Bayesian way of viewing the problem, as a byproduct also reproduces the bound obtained by \textcite{abernethy2016perturbation} for FTPL.

With this sanity-check completed, we proceed to study the continuous setting, where the Bayesian view of the problem tells us that we should choose $\gamma$ to be a strong equalizing adversary.
Motivated by a functional Central Limit Theorem, the natural choice where $Y$ is a bounded Lipschitz unit ball is the Matérn kernel with smoothness $1/2$.
To proceed, we first derive the expected supremum bounds needed for bounding prior regret. 
This follows by applying a standard chaining argument to bound the expected supremum in terms of the Dudley entropy integral, which we analyze as follows.

\begin{lemma}
\label{lem:matern_sup}
Let $X = [0,1]^d$, and let $\gamma\~[GP](0,k)$ where $k$ is Matérn-1/2 with variance $\sigma^2$ and length scale $\kappa$.
Then
\[
\E\sup_{x\in[0,1]^d}  \gamma(x) \leq 16 \sigma\sqrt{d\log\del{1+\frac{\sqrt{d}}{\kappa}}}
.
\]
Moreover, $\gamma$ admits a expected global modulus of continuity of $\psi(h) = \c{O}\del{\sqrt{h \log\del{\frac{1}{h}}}}$.
\end{lemma}

\begin{proof}
Note first that the result for general $\sigma>0$ follows immediately from the case $\sigma=1$, so assume this without loss of generality.
By the Dudley entropy integral bound \cite[Theorem 10.1]{lifshits12}, we have 
\[
\E \sup_{x\in[0,1]} \gamma(x) \leq 4\sqrt{2} \int_0^{1/2} \sqrt{\log C^{(d_\kappa)}_\eps([0,1]^d)} \d\eps
\]
where $C^{(d_\kappa)}_\eps([0,1]^d)$ is the minimum number of closed balls of radius $\eps$ needed to cover $[0,1]^d$ with respect to the \emph{kernel distance} on $[0,1]^d$, defined as
\[
d_\kappa(x,x')^2 = \E(\gamma(x) - \gamma(x'))^2 = 2 - 2k(x,x')
\]
where the $\kappa$-subscript emphasizes dependence on length scale.
We bound $C^{(d_\kappa)}_\eps([0,1]^d)$ from above by exhibiting a cover: choose one consisting of $n_\eps$ points in each dimension, for a total of $N_\eps = n_\eps^d$ points, placed at 
\[
\cbr{\frac{2\ell+1}{2n_\eps} :  \ell\in [n_\eps]}^d
\]
where dependence of the cover size on $\eps$ is to be determined next.
These points form a cover of radius $\frac{\sqrt{d}}{2n_\eps}$ in the Euclidean norm, meaning that we can choose $n_\eps$ to ensure $C_\eps^{(\norm{\.})}([0,1]^d) \leq \ceil{\frac{\sqrt{d}}{2\eps}}^d \leq \del{\frac{\sqrt{d}}{\eps}}^d$.
Next, note that the kernel distance is monotone in the Euclidean norm.
For two points $x_1,x_2$ with $\norm{x_1-x_2} = \delta_\eps$, we have
\[
k(x_1,x_2) &= \exp\del{-\frac{\delta_\eps}{\kappa}}
&
d_{k_\kappa}(x_1,x_2) &= \sqrt{2 - 2 \exp\del{-\frac{\delta_\eps}{\kappa}}}
\]
To get $d_{k_\kappa}(x_1,x_2) \leq \eps$, we must have $2 - 2 \exp\del{-\frac{\delta_\eps}{\kappa}} \leq \eps^2$, and solving for $\delta_\eps$ gives $\delta_\eps \leq -\kappa \log\del{1 - \frac{\eps^2}{2}}$. 
Let $\delta_\eps = \min\del{\sqrt{d},\kappa\log\del{1 - \frac{\eps^2}{2}}}$.
Then we have 
\[
C^{(d_\kappa)}_\eps([0,1]^d) \leq C^{(\norm{\.})}_{\delta_\eps}([0,1]^d) \leq \del{\frac{\sqrt{d}}{\delta_\eps}}^d &= \max\del{1,\frac{d^{d/2}}{\kappa^d\del{-\log\del{1 - \frac{\eps^2}{2}}}^d}} 
\\
&\leq \max\del{1,\frac{2^d d^{d/2}}{\kappa^d \eps^{2d}}}
\]
where the latter inequality follows from $-\log(1-x) \geq x$, which follows by convexity.
We use this to bound the Dudley entropy integral.
Let $\psi = \frac{\sqrt{d}}{\kappa}$ and $m=\min\del{\frac{1}{2}, \sqrt{2\psi }}$.
Write
\[
\int_0^{m} \sqrt{\log\del{\frac{2\psi}{\eps^2}}} \d \eps &= \sbr{\eps \sqrt{d\log\del{\frac{2\psi}{\eps^2}}} - \sqrt{d\pi\psi} \cdot \f{erf} \del{\sqrt{\frac{1}{2}\log\del{\frac{2\psi}{\eps^2}}}}}_0^m 
\\
&= 
\begin{cases}
\frac{1}{2}\sqrt{d \log(8\psi)} + \sqrt{d \pi \psi }\del{1-\f{erf}\del{\sqrt{\frac{1}{2} \log\del{ 8 \psi}}}} & : \psi\geq \frac{1}{8}
\\
\sqrt{d \pi \psi}\del{1-\f{erf}(0)}& : \psi<\frac{1}{8}
\end{cases} 
\\
& = 
\begin{cases}
\frac{1}{2} \sqrt{ d \log( 8 \psi)} + \sqrt{d \pi \psi}\del{1-\f{erf}\del{\sqrt{\frac{1}{2} \log\del{8\psi}}}} & : \psi \geq \frac{1}{8}
\\
\sqrt{d\pi \psi}& : \psi<\frac{1}{8}.
\end{cases}           
\]
where $\f{erf}$ is the error function. 
Using $1-\f{erf}(x)\leq \min\del{1, \frac{\exp(-x^2)}{x\sqrt{\pi}}}$, we get 
\[
\int_0^{m} \sqrt{\log\del{\frac{2\psi}{\eps^2}}} \d \eps \leq 
\begin{cases}
\frac{1}{2}\sqrt{d\log(8\psi)} + \min\del{\sqrt{d\pi\psi},\frac{\sqrt{d}}{2\sqrt{\log(8\psi)}}}& : \psi\geq \frac{1}{8} \\
\sqrt{d\pi\psi}& : \psi< \frac{1}{8}
\end{cases}        
\]
For $C$ large enough, it therefore follows that 
\[
\int_0^{m} \sqrt{\log C^{(d_\kappa)}_\eps([0,1]^d)} \d \eps &\leq  C\sqrt{\log(1+\psi)}
\]
where one can verify numerically that $C=2.7$ is sufficient.
Substituting this above, we get
\[
\E \sup_{x\in[0,1]} \gamma(x) \leq 16 \sqrt{d\log\del{1+\frac{\sqrt{d}}{\kappa}}}
\]
which gives the first claim.

We now examine the expected global modulus of continuity.
For this, define the mean-zero Gaussian process $\Delta(x,x') = \gamma(x) - \gamma(x')$ on the set $\cbr{(x,x')\in[0,1]^d  : \norm{x-x'}\leq h}$ of nearby pairs of points.
We have
\[
\E \sup_{\substack{x,x'\in[0,1]^d\\\norm{x-x'}\leq h}} |\gamma(x) - \gamma(x')| &= \E \sup_{\substack{x,x'\in[0,1]^d\\\norm{x-x'}\leq h}} |\Delta(x,x')| 
\\
&\leq 2\E \sup_{\substack{x,x'\in[0,1]^d\\\norm{x-x'}\leq h}} \Delta(x,x') + \inf_{\substack{x,x'\in[0,1]^d\\\norm{x-x'}\leq h}} \sqrt{\Var(\Delta(x,x'))}
\]
where the inequality is by \textcite[Proposition 10.2]{lifshits12}.
Since we have $\Delta(x,x) = 0$ deterministically by definition, it follows that the infimum over the process' variance is zero, so it suffices to bound first term.
From here, the argument is essentially a re-run of the preceding expected supremum bound, with a few key modifications worth noting:
\1 The space over which the supremum is taken is the rectangular strip defined by $\norm{x - x'} \leq h$, therefore the upper limit of the Dudley entropy integral is given by half the maximal standard deviation, which, for $\Delta$, is $\sqrt{1-\exp\del{-\frac{h}{\kappa}}} \leq \sqrt{\frac{h}{2\kappa}}$.
\2 The covariance satisfies $\Cov(\Delta(x_1,x'_1),\Delta(x_2,x'_2)) \leq 3 k(x_1,x_2) + k(x'_1,x'_2)$, allowing one to relate kernel distance covers with respect to $\Delta$ to those with respect to $\gamma$, which have been computed already in the expected supremum derivation. 
\0 
Omitting the calculation to reduce repetitiveness, we obtain a bound of
\[
\E \sup_{\substack{x,x'\in[0,1]^d\\\norm{x-x'}\leq h}} |\gamma(x) - \gamma(x')| \leq 32 \sqrt{\frac{dh}{2\kappa} \log \frac{20\sqrt{d}}{h}}
\]
which gives our desired global modulus of continuity.
\end{proof}

Applying the developed theory, we obtain the following.

\CorBLRegret*

\begin{proof}
We first note that Matérn kernels satisfy the assumptions needed, where the expected global modulus of continuity given above to be $\c{O}\del{\sqrt{h\log\del{\frac{1}{h}}}}$.
The key inequality is
\[
\sup_{y\in Y} y(x) - y(x') \frac{k(x,x')}{k(x',x')} \overset{\t{(i)}}&\leq \sup\cbr{u(x) : u(x') = 0, \norm{u}_\infty \leq \beta, \abs{u}_{\f{Lip}} \leq \lambda + \frac{\beta}{\kappa}}
\\
\overset{\t{(ii)}}&= \min\del{2\beta, \del{\lambda + \frac{\beta}{\kappa}} |x - x'| } 
\\
\overset{\t{(iii)}}&\leq \ubr{\frac{2\beta}{\del{1 - e^{-\frac{2\beta}{\lambda\kappa + \beta}}}}}{C_{Y,k}} \del{1 - \frac{k(x,x')}{k(x',x')}}
\]
where (i) follows by noting that the image of $Y$ under the map $y \mapsto y - y(x')\frac{k(x',\.)}{k(x',x')}$ is contained in the set of functions which are bounded by $2\beta$, have Lipschitz constant $\lambda + \frac{\beta}{\kappa}$, and go through zero at $x'$, (ii) follows because all functions in this class are upper-bounded by their Lipschitz envelope, which is given by the formula, and (iii) follows by explicit calculation from the functional form of the Matérn-1/2 kernel, where equality occurs when $|x'-x| = \frac{2\beta}{\lambda\kappa+\beta}$.
Taking $\kappa = \frac{\beta}{\lambda}$, and then $\sigma = \beta$, we obtain 
\[
R(p,q) &\leq \del{\sigma + \frac{\beta^2\del{1 + \frac{2}{1 - \frac{1}{e}}}}{\sigma}} \E_{x\in[0,1]^d} \sup_{x\in X} \gamma(x)
\\
&= \beta \del{32 + \frac{32}{1 - \frac{1}{e}}}\sqrt{T d\log\del{1 + \sqrt{d}\frac{\lambda}{\beta}}}
.
\]
\end{proof}

\end{document}